\documentclass{article}

\usepackage{microtype}
\usepackage{graphicx}
\usepackage{xcolor}
\usepackage{wrapfig}
\usepackage{booktabs} 

\usepackage{hyperref}

\usepackage[accepted]{icml2019}

\usepackage{amsmath,amsfonts,bm}

\def\Secref#1{Section~\ref{#1}}

\def\eqref#1{Eq.~(\ref{#1})}

\def\1{\bm{1}}

\def\va{{\bm{a}}}
\def\vb{{\bm{b}}}
\def\vc{{\bm{c}}}

\def\vh{{\bm{h}}}

\def\vp{{\bm{p}}}

\def\vv{{\bm{v}}}
\def\vw{{\bm{w}}}
\def\vx{{\bm{x}}}
\def\vy{{\bm{y}}}
\def\vz{{\bm{z}}}

\DeclareMathAlphabet{\mathsfit}{\encodingdefault}{\sfdefault}{m}{sl}
\SetMathAlphabet{\mathsfit}{bold}{\encodingdefault}{\sfdefault}{bx}{n}

\def\gD{{\mathcal{D}}}

\def\gF{{\mathcal{F}}}

\def\gK{{\mathcal{K}}}
\def\gL{{\mathcal{L}}}

\def\gN{{\mathcal{N}}}
\def\gO{{\mathcal{O}}}

\def\gT{{\mathcal{T}}}

\def\gX{{\mathcal{X}}}

\newcommand{\E}{\mathbb{E}}

\newcommand{\R}{\mathbb{R}}

\usepackage{graphicx}
\usepackage{courier}
\usepackage{color}
\usepackage{gensymb}
\usepackage{cite}
\usepackage{amsthm}
\usepackage{thm-restate}
\usepackage{enumitem}
\usepackage{mathtools}
\usepackage{soul}
\usepackage{algorithmic}
\usepackage[algo2e,vlined,ruled]{algorithm2e}
\usepackage{subfig}
\newtheorem{theorem}{Theorem}[section]
\usepackage{amsfonts}
\theoremstyle{definition}
\newtheorem{definition}{Definition}[section]
 
\theoremstyle{remark}
\newtheorem*{remark}{Remark}

\icmltitlerunning{Particle Flow Bayes' Rule}
\begin{document}

\twocolumn[
\icmltitle{Particle Flow Bayes' Rule}



\icmlsetsymbol{equal}{*}

\begin{icmlauthorlist}
\icmlauthor{Xinshi Chen}{equal,gt1}
\icmlauthor{Hanjun Dai}{equal,gt2}
\icmlauthor{Le Song}{gt2,ant}
\end{icmlauthorlist}

\icmlaffiliation{gt1}{School of Mathematics,}
\icmlaffiliation{gt2}{School of Computational Science and Engineering, Georgia Institute of Technology, Atlanta, Georgia, USA.}
\icmlaffiliation{ant}{Ant Financial, Hangzhou, China}

\icmlcorrespondingauthor{Xinshi Chen}{xinshi.chen@gatech.edu}

\icmlkeywords{ODE, Bayesian Inference, Particle Flow, Sequential Monte Carlo, ICML}

\vskip 0.3in
]



\printAffiliationsAndNotice{\icmlEqualContribution} 

\begin{abstract}
We present a particle flow realization of Bayes' rule, where an ODE-based neural operator is used to transport particles from a prior to its posterior after a new observation. We prove that such an ODE operator exists. Its neural parameterization can be trained in a meta-learning framework, allowing this operator to reason about the effect of an individual observation on the posterior, and thus generalize across different priors, observations and to sequential Bayesian inference. We demonstrated the generalization ability of our particle flow Bayes operator in several canonical and high dimensional examples. 
\end{abstract}

\vspace{-4mm}
\section{Introduction} \label{sec:intro}

\begin{figure*}[t!]
    \vspace{-2mm}
    \centering
    \includegraphics[width=0.8\textwidth]{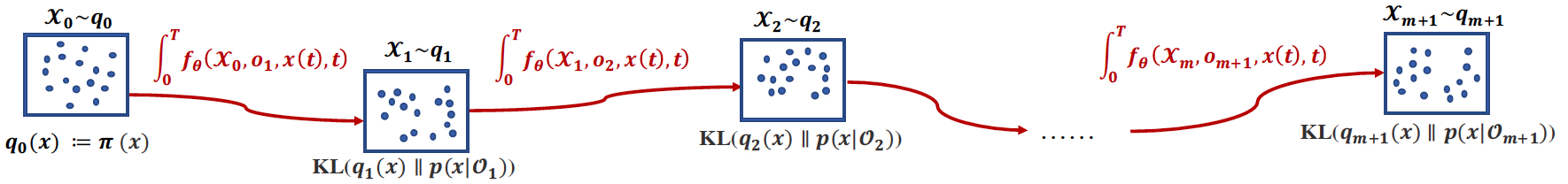}
    \vspace{-2.5mm}
    \caption{\small PFBR framework: sequential Bayesian inference as a deterministic flow of particles.}
    \label{fig:seq-framework}
    \vspace{-4mm}
\end{figure*}

\setlength{\abovedisplayskip}{4pt}
\setlength{\abovedisplayshortskip}{1pt}
\setlength{\belowdisplayskip}{4pt}
\setlength{\belowdisplayshortskip}{1pt}
\setlength{\jot}{3pt}
\setlength{\textfloatsep}{6pt}  

In many data analysis tasks, it is important to estimate unknown quantities $\vx \in \R^d$ from observations $\gO_m:=\{o_1,\cdots,o_m\}$. Given prior knowledge $\pi(\vx)$ and likelihood functions $p(o_t|\vx)$, the essence of Bayesian inference is to compute the posterior $p(\vx|\gO_m)\propto \pi(\vx)\prod_{t=1}^{m} p(o_t|\vx)$ by Bayes' rule. For many nontrivial models, the prior might not be conjugate to the likelihood, making the posterior not in a closed form. Therefore, computing the posterior often results in intractable integration and poses significant challenges. Typically, one resorts to approximate inference methods such as sampling (e.g., MCMC)~\citep{AndFreDouJor03} or variational inference~\citep{WaiJor03}.

In many real problems, observations arrive sequentially online, and Bayesian inference needs be performed recursively, 
\vspace{-2.5mm}
\begin{equation}\label{eq:recursive-Bayes}
     \overbrace{p(\vx|\gO_{m+1})}^\text{updated posterior}~\propto \overbrace{p(\vx|\gO_{m})}^\text{current posterior}\overbrace{p(o_{m+1}|\vx)}^\text{likelihood}. 
\end{equation}
That is the estimation of $p(\vx|\gO_{m+1})$ should be computed based on the estimation of $p(\vx|\gO_{m})$ obtained from the last iteration and the presence of the new observation $o_{m+1}$. 
It therefore requires algorithms which allow for efficient online inference. In this case, 
both standard MCMC and variational inference become inefficient, since the former requires a complete scan of the dataset in each iteration, and the latter requires solving an optimization for every new observation. Thus, sequential Monte Carlo (SMC)~\citep{DouFreGor01,BalMad06} or stochastic approximations, such as stochastic gradient Langevin dynamics~\citep{WelTeh11} and stochastic variational inference~\citep{HofBleWanPai12}, are developed to improve the efficiency. However, SMC suffers from the path degeneracy problems in high dimensions~\citep{DaumHuang03,SnyderBeng08}, and rejuvenation steps are designed but may violate the online sequential update  requirement~\citep{CanShiGri09,ChopinJacob13}. 
Stochastic approximation methods are prescribed algorithms that cannot exploit the structure of the problem for further improvement.

To address these challenges, the seminal work of Kernel Bayes Rule (KBR) views the Bayes update as an operator in reproducing kernel Hilbert spaces (RKHS) which can be learned and directly produce the posterior from prior after each observation~\citep{FukSonGre12}. In the KBR framework, the posterior is represented as an embedding $\mu_m:=\mathbb{E}_{p(\vx|\gO_m)}[\phi(\vx)]$ using a feature map $\phi(\cdot)$ associated with a kernel function; then the kernel Bayes operator $\gK(\cdot, o)$ will take this embedding as input and produce the embedding of the updated posterior, 
\begin{equation}
    \overbrace{\mu_{m+1}}^{\text{updated embedding}} =\quad \gK(\overbrace{\mu_m}^{\text{current embedding}},\quad o_{m+1}\quad). 
\end{equation}
Another novel aspect of KBR method is that it contains a training phase and a testing phase, where the structure of the problem at hand (e.g., the likelihood) is taken into account in the training phase, and in the testing phase, the learned operator $\gK$ will directly operate on the current posterior $\mu_m$ to produce the output. However, despite the nice concepts of KBR operator, it only works well for a limited range of problems due to its strong theoretical assumptions.

In this work, we aim to lift the limitation of KBR operator, and will design a novel continuous particle flow operator $\gF$ to realize the Bayes update, for which we call it {\bf particle flow Bayes' rule} (PFBR). In the PFBR framework (Fig.~\ref{fig:seq-framework}), a prior distribution $\pi(\vx)$, or, the current posterior $\pi_m(\vx):=p(\vx|\gO_m)$ is approximated by a set of $N$ {equally weighted} particles $\gX_m = \{\vx_m^1,\ldots,\vx_m^N\}$; 
and then, given an observation $o_{m+1}$, the flow operator $\gF(\gX_m, \vx_m^n, o_{m+1})$ will transport each particle $\vx_m^n$ to a new particle $\vx_{m+1}^n$ to approximate the new posterior $p(\vx|\gO_{m+1})$. That is,
\begin{align}
\overbrace{\vx_{m+1}^n}^{\text{updated particle}}
&=\quad \gF({\gX_m}, \overbrace{\vx_m^n}^{\text{current particle}},\quad o_{m+1}),
\end{align}
where $\gX_{m+1} = \{\vx_{m+1}^1, \ldots, \vx_{m+1}^N\}$ will be used as samples from the new posterior $p(\vx|\gO_{m+1})$. Furthermore, this PFBR operator $\gF$ can be applied recursively to $\gX_{m+1}$ and a new observation $o_{m+2}$ to produce $\gX_{m+2}$, and so on. 

In a high-level, we model the PFBR operator $\gF$ as a continuous deterministic flow, which propagates the locations of particles and the values of their probability density simultaneously through a dynamical system described by ordinary differential equations (ODEs). 
A natural question is whether a fixed ODE-based Bayes operator applicable to different prior distributions and likelihood functions exists. In our paper, we resolve this important theoretical question by making a novel connection between PFBR operator and the Fokker-Planck equation of Langevin dynamics. The proof of existence also provides a basis for our parameterization of PFBR operator using DeepSet~\citep{ZahKotRavPocetal17}.  

Similar to KBR, PFBR have a training phase and a testing phase. However, the training procedure is very different as it adopts a meta-learning framework~\citep{AndDenGomHofetal16}, where multiple related Bayesian inference tasks with different priors and observations are created. PFBR operator $\gF$ will learn from these tasks how to update the posteriors given new observations. During test phase, the learned PFBR will be applied directly to new observations without either re-optimization or storing previous observations. We conduct various experiments to show that the learned PFBR operator can generalize to new Bayesian inference tasks. 

{\bf Related work.} There is a recent surge of interests in ODE-based Bayesian inference~\citep{ChenRubanova18,ZhaEWan18,GraCheBetSutetal18,LeiSuYauGu17}. These works focus on fitting a single target distribution. Consequently, the learned flow can not generalize directly to a new dataset, a new prior or to sequential setting without re-optimization. 

\section{Bayesian Inference as Particle Flow}

We present details in this section from four aspects: (1) How to map sequential Bayes inference to particle flow? (2) What is the property of such particle flow? (3) Does a shared flow-based Bayesian operator $\gF$ exist? (4) How to parameterize the flow operator $\gF$? 

\vspace{-1mm}
\subsection{PFBR: Particle Flow Bayes' Rule}

The problem mapping from sequential Bayesian inference to particle flow goes as follows. Initially, $N$ particles $\gX_0=\{\vx_0^1,\ldots,\vx_0^N\}$ are sampled i.i.d. from a prior $ \pi(\vx)$. Given an observation $o_1$, the operator $\gF$ will transport the particles to $\gX_{1}=\{\vx_1^1,\ldots,\vx_1^N\}$ to estimate the posterior $p(\vx|\gO_1)\propto \pi(\vx)p(o_1|\vx)$. We define this transformation as the solution of an ODE. That is, $\forall n$,
\begin{equation*}
    \left\{\begin{tabular}{l}
        $\frac{d\vx}{dt} = f(\gX_0,o_{1},\vx(t),t),~\forall t\in(0,T]$ \cr
        $\vx(0) = \vx_0^n$
    \end{tabular}\right.
    \xRightarrow{\text{gives}}\vx_1^n =\vx(T).
\end{equation*}
The flow velocity $f$ takes observation $o_1$ as input and determines both direction and speed of the change of $\vx(t)$. In this ODE model, each particle $\vx_0^n$ sampled from the prior gives an initial value $\vx(0)$, and then the flow velocity $f$ will evolve the particle continuously and deterministically. At terminate time $T$, we will take solution $\vx(T)$ as the transformed particle $\vx_1^n$ for estimating the posterior. 

Applying this ODE-based transformation sequentially as new observations $o_2, o_3,\ldots$ arrive, we can define a recursive particle flow Bayes operator, called PFBR, as
\begin{align}
    \vx_{m+1}^n 
    & =\gF(\gX_m, o_{m+1}, \vx_m^n) \nonumber \\
    &:=\vx_m^n +\textstyle{ \int_{0}^{T}}f(\gX_m, o_{m+1}, \vx(t),t)\,dt. \label{eq:flow}
\end{align}
The set of obtained particles $\gX_{m+1}$ can be used to perform Bayesian inference such as estimating the mean and quantifying the uncertainty of any test function by averaging over these particles.

At this moment, we assume $f$ has a form of $f(\gX,o, \vx(t),t)$, and will be shared across different sequential stages. In section~\ref{sec:existence}, a rigorous discussion on the existence of a shared flow velocity of this form will be made. 
Next, we will discuss further properties of this continuous particle flow which will help us study the existence of such operator for Bayesian inference, and design the parameterization and the learning for the flow velocity.

\vspace{-1mm}
\subsection{Property of Continuous Deterministic  Flow}\label{sec:mass-transport}
\vspace{-0.5mm}

The continuous transformation of $\vx(t)$ described by ODE $d\vx/dt = f$ defines a \emph{deterministic} flow for each particle. Let $q(\vx, t)$ be the probability density of the continuous random variable $\vx(t)$. The change of this density is also determined by $f$. More specifically, $q$ follows the {\bf continuity equation}~\citep{Batchelor00}:
\begin{align}\label{eq:continuity}
    \partial q(\vx,t)/\partial t = -\nabla_x\cdot (qf),
\end{align}
where $\nabla_x\cdot$ is the divergence operator.
Continuity equation is the mathematical expression for the law of {\it local conservation of mass} - mass can neither be created nor destroyed, nor can it ''teleport'' from one place to another.

Given continuity equation, one can describe the change of log-density by another differential equation (Theorem~\ref{thm:change-of-log-density}). 

\begin{theorem}\label{thm:change-of-log-density} If $d\vx/dt =f$, then the change in log-density follows the differential equation~\citep{ChenRubanova18}
\begin{equation}
    d \log(q(\vx, t))/dt = -\nabla_x \cdot f.
\end{equation}
\end{theorem}
\vspace{-1mm}

Since for any physical quantity $q(\vx,t)$, the distinguish between material derivative $d/dt$ and partial derivative $\partial/\partial t$ is important, we clarify the definition before the proof of this theorem.

\begin{definition} {\it Material derivative} of $q(\vx,t)$ is defined as
\begin{equation}
    dq/dt = \partial q/\partial t +  \nabla_x q\cdot d\vx/dt.
\end{equation}
Note that $dq/dt$ defines the rate of change of $q$ in a given particle as it moves along its trajectory $\vx=\vx(t)$ in the flow, while $\partial q/\partial t$ means the rate of change of $q$ at a particular point $\vx$ that is fixed in the space. 
\end{definition}
\begin{proof}[Proof of Theorem~\ref{thm:change-of-log-density}]
\vspace{-2mm}
By continuity equation,
$\frac{\partial q}{\partial t} =- \nabla_x q\cdot f -q\nabla_x\cdot f \Rightarrow \frac{d q}{dt} = - q \nabla_x\cdot f $. By chain rule, we have
$\frac{d\log q}{dt} = \frac{1}{q}\frac{dq}{dt} = \frac{1}{q}(-q \nabla_x\cdot f ) = -\nabla_x \cdot f$.
\end{proof}
\vspace{-2mm}
Theorem~\ref{thm:change-of-log-density} gives the same result as the Instantaneous Change of Variables Theorem stated by~\citet{ChenRubanova18}. However, our statement is more accurate using the notation of material and partial derivatives. Our proof is simpler and intuitively clearer using continuity equation. This also helps us to see the connection to other physics problems such as fluid dynamics and electromagnetism.

With theorem~\ref{thm:change-of-log-density}, we can compute the log-density of the particles by integrating across $(0,T]$ for each $n$:
\begin{align}
    \label{eq:density_evolve}
   \log q_{m+1}(x_{m+1}^n) = \log q_m(x_m^n) - \textstyle{\int_{0}^{T}}\nabla_x \cdot f\,dt. 
\end{align}

\vspace{-1mm}
\subsection{Existence of Flow-based Bayes' Rule}\label{sec:existence}
\vspace{-0.5mm}

Does a unified flow velocity $f$ exist for different Bayesian inference tasks involving different priors and observations? If it does, what is the form of this function? These questions
are non-trivial even for simple Gaussian case. 

For instance, let the prior $\pi(x) = \gN(0, \sigma_x)$ and the likelihood $p(o|x) = \gN(x, \sigma)$ both be one dimensional Gaussian distributions. Given an observation $o=0$, the posterior distribution is $\gN(0, (\sigma \cdot \sigma_x)/(\sigma+\sigma_x))$. It means the ODE $d\vx/dt = f$ needs to push a zero mean Gaussian distribution with covariance $\sigma_x$ to another zero mean Gaussian distribution with covariance $(\sigma \cdot \sigma_x)/(\sigma+\sigma_x)$ for any $\sigma_x$. It is not clear whether such a unified flow velocity function $f$ exists and what is the form for it. 

To resolve the existence issue, we will first establish a connection between the deterministic flow in~\Secref{sec:mass-transport} and the stochastic flow: Langevin dynamics. Then we will leverage the connection between closed-loop control and open-loop control to show the existence of a unified $f$.

\vspace{-1mm}
\subsubsection{Connection to Stochastic Flow}
\vspace{-1mm}

Langevin dynamics is a \emph{stochastic} process 
\begin{align}
    \label{eq:Langevin}
    d\vx(t) = &\nabla_x\log p (\vx|\gO_{m})p (o_{m+1}|\vx)\,dt +\sqrt{2}d\vw(t), 
\end{align}
where $d\vw(t)$ is a standard Brownian motion. Given a fixed initial location $\vx(0)$, multiple runs of the Langevin dynamics to time $t$ will result in multiple random locations of $\vx(t)$ due to the randomness of $\vw(t)$. This stochastic flow is very different in nature comparing to the deterministic flow in \Secref{sec:mass-transport}, where a fixed location $\vx(0)$ will always end up with the same location $\vx(t)$. 

Nonetheless, while Langevin dynamics is a stochastic flow of a continuous random variable $\vx(t)$, the probability density $q(\vx,t)$ of $\vx(t)$ follows a \emph{deterministic} evolution according to the associated {\bf Fokker-Planck equation}~\citep{JordanOtto98}
\begin{align}\label{eq:FK-PLK}
    {\partial q}/{\partial t}= &-\nabla_x\cdot\left(q\nabla_x\log p (\vx|\gO_{m})p (o_{m+1}|\vx)\right)\nonumber \\
    &+ \Delta_x q(\vx,t),
    \end{align}
where $\Delta_x = \nabla_x\cdot \nabla_x$ is the Laplace operator. 
Furthermore, if the so-called potential function $\Psi(\vx):= -\log p (\vx|\gO_{m})p (o_{m+1}|\vx )$ is smooth and $e^{-\Psi}\in L^1(\R^d)$, the Fokker-Planck equation has a unique stationary solution in the form of a Gibbs distribution~\citep{JordanOtto98},
\begin{equation}
    q(\vx, \infty) =e^{-\Psi}/Z = p (\vx|\gO_{m})p (o_{m+1}|\vx)/Z.
\end{equation}
Clearly, this stationary solution is the posterior distribution $p(\vx|\gO_{m+1})$.

Now we will rewrite the Fokker-Planck equation in~\eqref{eq:FK-PLK} into the form of the deterministic flow in~\eqref{eq:continuity} from \Secref{sec:mass-transport}, and hence identify the corresponding flow velocity. 

\begin{theorem}\label{thm:station-solution}
Assume the deterministic transformation of a continuous random variable $\vx(t)$ is $d\vx/dt=f$, where
\begin{align}
\label{eq:fkplk-flow}
f= \nabla_x\log p (\vx|\gO_{m})p (o_{m+1}|\vx )-\nabla_x \log q(\vx,t)
\end{align}
and $q(\vx,t)$ is the probability density of $\vx(t)$.
If the potential function $\Psi$ is smooth and  $e^{-\Psi}\in L^1(\R^d)$, then $q(\vx,t)$ converges to $p(\vx|\gO_{m+1})$ as $t\rightarrow \infty$.
\vspace{-2mm}
\end{theorem}
\begin{proof}
By continuity equation in~\eqref{eq:continuity}, the probability density $q(\vx, t)$ of $\vx(t)$ satisfies
$
{\partial q}/{\partial t} = -\nabla_x\cdot \big(q\big(-\nabla_x\Psi(\vx)-\nabla_x \log q(\vx,t)\big) \big)
$. It is easy to see this equation is the same as the Fokker-Planck equation in~\eqref{eq:FK-PLK}, by decomposing the Laplace as $\Delta_x q = \nabla_x \cdot (q\nabla_x \log q) $. Under the conditions for the potential function $\Psi$, since Fokker-Planck equation has a unique stationary distribution $q(\vx,\infty)$ equal to the posterior distribution $p(\vx|O_{m+1})$, the deterministic flow in~\eqref{eq:fkplk-flow} will also converge to $p(\vx|O_{m+1})$. 
\vspace{-3mm}
\end{proof}
The implication of Theorem~\ref{thm:station-solution} is that we can construct a deterministic flow of particles to obtain the posterior and hence establish the existence. However, the flow velocity in~\eqref{eq:fkplk-flow} depends on the intermediate density $q(\vx,t)$ which changes over time. This seemingly suggests that $f$ can not be expressed as a fixed function of $p(\vx|\gO_{m})$ and $p(o_{m+1}|\vx)$. In the next section, we show that this dependence on $q(\vx,t)$ can be removed using theory of optimal control for deterministic systems.  

\subsubsection{Closed-loop to Open-loop Conversion}

Now the question is: whether the term  $\nabla_x \log q(\vx,t)$ in \eqref{eq:fkplk-flow} can be made independent of $q(\vx,t)$, or whether there is a equivalent form which can achieve the same flow. To investigate this question, we consider the the following \emph{deterministic} optimal control problem
\begin{align}\label{eq:opt-control1}
    \min_w~~&d(q(\vx,\infty),p(\vx|\gO_{m+1}) )\\
    \text{s.t.}~~&\frac{d\vx}{d t}=\nabla_x \log p(\vx|\gO_m)p(o_{m+1}|\vx)- w,\label{eq:opt-control2}
\end{align}
where $d$ can be any metric defined on the set of densities over $\R^d$. 
In~\eqref{eq:opt-control1}, we are optimizing over $w$, which is usually called the {\it control}. By Theorem~\ref{thm:station-solution}, $w=\nabla_x \log q(\vx,t)$ is apparently an optimal solution. Furthermore, the corresponding flow velocity derived by Fokker-Planck equation can be regarded as the continuous steepest descent of KL$(q(\vx,t)||p(\vx|\gO_{m+1}))$ under Wasserstein distance~\citep{JordanOtto98}. We are seeking an alternative expression to the above optimal solution which only depends on $p(\vx,\gO_m)$ and $p(o_{m+1}|\vx)$. First, we introduce the terminology below from optimal control literature. 

\begin{definition}
In optimal control literature, a control in a feed-back form $w=w(q(\vx,t),t)$ is called {\bf closed-loop}. In contrast, another type of control $w=w(q(\vx,0), t)$ is called {\bf open-loop}. An open-loop control is determined when the initial state $q(\vx,0)$ is observed, whereas, a closed-loop control can adapt to the encountered states $q(\vx,t)$.
\end{definition}

\begin{restatable}{theorem}{thmexist}
\label{th:open_loop}
For the optimal control problem in~\eqref{eq:opt-control1} and~\eqref{eq:opt-control2}, there exists an open-loop control $w^* = w^*(q(\vx,0),t)$ such that the induced state $q^*(\vx,t)$ satisfies $q^*(\vx,\infty)=p(\vx|\gO_{m+1})$. Moreover, $w^*$ has a fixed expression with respect to  $p(\vx|\gO_m)$ and $p(o_{m+1}|\vx)$ across different $m$.
\end{restatable}
\vspace{-4mm}
\begin{proof}
By Theorem~\ref{thm:station-solution}, $\tilde{w}^*(q(\vx,t),t) := \nabla_x \log q(\vx,t)$ can induce the optimal state $\tilde{q}^*(\vx,\infty)=p(\vx|\gO_{m+1})$ and achieve a zero loss, $d = 0$. Hence, $\tilde{w}^*$ is an optimal closed-loop control for this problem.

Although in general closed-loop control has a stronger characterization to the solution, in a deterministic system like~\eqref{eq:opt-control2}, the optimal closed-loop control and the optimal open-loop control will give the same control law and achieve the same optimal loss~\citep{Dreyfus64}. Hence, there exists an optimal open-loop control $w^
*=w^*(q(\vx,0),t)$ such that the induced state also gives a zero loss and thus $q^*(\vx,\infty)=p(\vx|\gO_{m+1})$. More details are provided in Appendix~\ref{apx:existence} to express $w^*$ as a fixed function of $q(\vx,0)$, $p(\vx|\gO_m)$, $p(o_{m+1}|\vx)$ and $t$.
\end{proof}
\vspace{-2mm}

{\bf Conclusion of a unified $f$.}
In sequential Bayesian inference, we will set $q(\vx,0)$ as $p(\vx|\gO_{m})$. Therefore, Theorem~\ref{th:open_loop} shows that there exists a fixed and deterministic flow velocity $f$ of the form 
\begin{align}\label{eq:shared-f}
  \nabla_x\log p (\vx|\gO_{m})p (o_{m+1}|\vx)
  -w^*(p (\vx|\gO_{m}),t),
\end{align}
which can transform $p(\vx|\gO_m)$ to $p(\vx|\gO_{m+1})$ and in turns define a unified particle flow Bayes operator $\gF$.

\subsection{Parametrization}
 We design a practical parameterization of $f$ based on the expression of the unified flow $f$ in~\eqref{eq:shared-f}.

{(i) $p(\vx|\gO_m)\Rightarrow \gX_m$:}
Since we do not have full access to the density $p(\vx|\gO_{m})$ but have samples $\gX_m = \{\vx_m^1,\ldots,\vx_m^{N}\}$ from it, we can use these samples as surrogates for $p(\vx|\gO_{m})$. A related example is feature space embedding of distributions~\citep{SmoGreSonSch07}, 
$\mu_{\gX}(p):= \textstyle{\int_{\gX}}\phi(\vx)p(\vx)\,d\vx \approx \textstyle{\frac{1}{N}\sum_{n=1}^N}\phi(\vx^n),~\vx^n\sim p$. 
Ideally, if $\mu_{\gX}$ is an injective mapping from the space of probability measures over $\gX$ to the feature space, the resulting embedding can be treated as a sufficient statistic of the density and any information we need from $p(\vx|\gO_m)$ can be preserved. Hence, we represent $p(\vx|\gO_m)$ by $\frac{1}{N}\sum_{n=1}^N \phi(\vx_m^n)$, where $\phi(\cdot)$ is a nonlinear feature mapping to be learned. Since we use a neural version of $\phi(\cdot)$, 
this representation can also be regarded as a DeepSet~\citep{ZahKotRavPocetal17}.

{(ii) $p(o_{m+1}|\vx)\Rightarrow (o_{m+1}, \vx(t))$:}
In both Langevin dynamics and~\eqref{eq:shared-f}, the only term containing the likelihood is $\nabla_x \log p(o_{m+1}|\vx)$. Consequently, we can use this term as an input to $f$. In the case when the likelihood function is fixed, we can also simply use the observation $o_{m+1}$, which results in similar performance in our experiments.

 Overall we parameterize the flow velocity as 
\begin{align}\label{eq:interacting}
    f = \vh\left( \textstyle{\frac{1}{N}\sum_{n=1}^N} \phi(\vx_m^n),o_{m+1},\vx(t),t \right ),
\end{align}
where $\vh$ and $\phi$ are neural networks (See specific architecture we use in Appendix~\ref{apx:parameterize}). Let  $\theta\in\Theta$ be their parameters which are independent of $t$. From now on, we write $f = f_{\theta}(\gX_m, o_{m+1}, \vx(t),t)$. In the next section, we will propose a meta learning framework for learning these parameters. 

\section{Learning Algorithm}

Since we aim to learn a generalizable Bayesian operator, we need to create multiple inference tasks as the training set and design the corresponding learning algorithm.

{\bf Multi-task Framework.}
The training set $\gD_{train}$ contains multiple inference tasks.
Each task is a tuple
\begin{align*}
   \gT := \left(\pi(\vx), p(\cdot | \vx), \gO_M:=\{o_1,\ldots,o_M\}\right) \in \gD_{train}
\end{align*}
which consists of a prior distribution, a likelihood function and a sequence of $M$ observations. A task with $M$ sequential observations can also be interpreted as a sequence of $M$ sub-tasks with 1 observation:
 \begin{align*}
    \tau := (p(\vx|\gO_m), p(\cdot|\vx), o_{m+1}) \in (\pi(\vx), p(\cdot | \vx), \gO_M).
 \end{align*}
Therefore, each task is a sequential Bayesian inference and each sub-task corresponds to one step Bayesian update.

{\bf Cumulative Loss Function.}
For each sub-task we define a loss $\text{KL}(q_m(\vx)||p(\vx|\gO_{m+1}))$, where $q_m(\vx)$ is the distribution transported by $\gF$ at $m$-th stage and $p(\vx|\gO_{m+1})$ is the target posterior (see Fig.~\ref{fig:seq-framework} for illustration). Meanwhile, the loss for the corresponding sequential task will be $
 \textstyle{\sum_{m=1}^M}\text{KL}(q_m(\vx)||p(\vx|\gO_m) )
$, which sums up the losses of all intermediate stages. Since its optimality is independent of normalizing constants, it is equivalent to minimize the negative evidence lower bound (ELBO)
 \begin{align}\label{eq:elbo}
  \gL(\gT) = {\sum_{m=1}^M} \sum_{n=1}^N \left(\log q_m(\vx_m^n)-\log p(\vx_m^n,\gO_m)\right).
 \end{align}
 \vspace{-3mm}

 The above expression is an empirical estimation using particles $\vx_m^n$. In each iteration during training phase, we will randomly sample a task from $\gD_{train}$ and update the PFBR operator $\gF$ by the loss gradient.

\subsection{Training Tasks Creation}\label{sec:task-creation}

Similar to some meta learning problems, the distribution of training tasks will essentially affect learner's performance on testing tasks~\citep{DaiKhaZhaDiletal17}. Depending on the nature of the Bayesian problem, we propose two approaches to construct the multitask training set: one is data driven, and the other is based on generative models. The general principle is that the collection of training priors have to be diverse enough or representative of those may be seen in testing time.

{\bf Data-Driven Approach.}
We can use the posterior distribution obtained from last Bayesian inference step as the new prior distribution. If the posterior distribution has an analytical (but unnormalized) expression, we will directly use this expression. 

More precisely, since each Bayesian inference step will generate a set of particles, $\gX_m = \{\vx_m^1,\ldots, \vx_m^N\}$, corresponding to the posterior distribution $p(\vx|\gO_m)$, we can apply a kernel density estimator on top of these samples to obtain an empirical density function \begin{align}\label{eq:kde_prior}
    \hat{\pi}(\vx; \gX_m) =\textstyle{ \frac{1}{N} \sum_{n=1}^N \frac{1}{\sqrt{2\pi}\sigma^d} }
    e^{ - \frac{\| \vx - \vx_m^n \|^2}{2\sigma^2} 
        }, 
\end{align}
where $\sigma$ is the kernel bandwidth. Then this density function with a set of samples from it will be used as the new prior for the next training task. This approach has two attractive properties. First, it does not require any prior knowledge of the model and thus is generally applicable to most problems. Second, it provides us a way of breaking a long sequence $(\pi(\vx), p(\cdot | \vx), \gO_M)$ with large $M$ into multiple tasks with shorter sequences
\begin{align}
\label{eq:seq_break} 
(\pi(\vx), p(\cdot|\vx), o_{1:m}) \cup (\hat{\pi}(\vx;\gX_m), p(\cdot|\vx). o_{m+1:M}),
\end{align}
This will help make the training procedure more efficient. This approach will be particularly suited to sequential Bayesian inference and is used in later experiments.

{\bf Generative Model Approach.}
Another possibility is to sample priors from a flexible generative model, such as a Dirichlet process Gaussian mixture model~\citep{Antoniak74}. We will leave the experimental evaluation of this approach for future investigation.

\subsection{Overall Learning Algorithm}

Learning the PFBR operator $\gF$ is learning the parameters $\theta$ of $f_\theta$ to minimize the following loss for multiple tasks 
\begin{align*}
    \gL(\gD_{train}) = \frac{1}{|\gD_{train} |} \sum_{\gT \in \gD_{train}} \gL(\gT), 
\end{align*}
where $\gL(\gT)$ is the loss for a single task defined in~\eqref{eq:elbo}. 
Both the particle location $\vx_m^n$ and its density value $q_m(\vx_m^n)$ in $\gL(\gT)$ are results of forward propagation determined by $f_{\theta}$ according to ODEs in~\eqref{eq:flow} and~\eqref{eq:density_evolve}. The training procedure is similar to other deep learning optimizations, except that an ODE technique called adjoint method~\citep{ChenRubanova18} is used to compute the gradient $d \gL/d\theta$ with very low memory cost. The overall algorithm under a meta-learning framework is summarized in Algorithm~\ref{algo:meta}.

\begin{algorithm}[t!]
  \DontPrintSemicolon
  \SetKwFunction{Grad}{Grad}
  \SetKwProg{Fn}{Function}{:}{}
  \SetKwFor{uFor}{For}{do}{}
  \SetKwFor{ForPar}{For all}{do in parallel}{}
  \SetKwComment{Comment}{$\triangleright$\ }{}
  \SetCommentSty{mycommfont}

  $\theta \gets$ random initialization \;
  
  \uFor{$itr=1$ to \#iterations}{
  $\gT = (\pi(\vx), p(\cdot|\vx), \gO_M) \gets$~sampled from $\gD_{train}$ \;
  
  $\gX_0 = \{\vx_0^n\}_{n=1}^N \overset{i.i.d.}{\sim} \pi(\vx)$\; \Comment*[r]{initial particles} 
  $\theta \gets \theta - \eta\, $\Grad{$\theta,\gX_0,\pi(\vx),p(\cdot|\vx),\gO_M$}\; 
  
  \uIf{{\rm mod}$(itr, k)=0$}{
  Perform a validation step on $\gD_{vali}$\; \Comment*[r]{validation}
  }
  }
\KwRet best $\theta^*$ in validation steps \;

\begin{remark}
See Appendix~\ref{apx:adjoint-method} for both derivation and algorithm steps of \Grad{}.
\end{remark}

\caption{Overall Learning Algorithm}\label{algo:meta}

\end{algorithm}

\subsection{Efficient Learning Tricks}\label{sec:efficient-training}

We introduce two techniques to improve training efficiency for large scale problems. Its application to an experiment which contains millions of data points is demonstrated in Section~\ref{sec:experiments}. These two techniques can be summarized as {\bf mini-batch embedding} and {\bf sequence-segmentation}. 

The loss function $\gL(\gT)$ in~\eqref{eq:elbo} contains a summation from $m=1$ to $M$ and also a hidden inner-loop summation 
\begin{align}\label{eq:joint-density}
\hspace{-3mm}
    \log p(\vx_m^n, \gO_m) =\log \pi(\vx_m^n)+\textstyle{\sum_{t=1}^m} \log p(o_t|\vx_m^n). 
\end{align}
 Thus the evaluation cost of $\gL(\gT)$ is quadratic with respect to the length $M$ of the observation sequences. Therefore, we need to reduce the length for large scale problems.

{\bf Mini-batch embedding.} Previously we defined $o_{m}$ as a single observation. However, we can also view it as a batch of $L$ observations,~i.e., $o_m=\{o_m^1,\ldots, o_m^L \}$. Each Bayesian update corresponding to this mini-batch will become $p(\vx|\gO_{m+1})\propto p(\vx|\gO_m)\prod_{l=1}^L p(o_{m}^l|\vx)$. If we rewrite $p(o_m|\vx)=\prod_{l=1}^L p(o_{m}^l|\vx)$, this is essentially the same as our previous expression. Therefore, we can replace $o_m$ by $\{o_m^l\}_{l=1}^L$ and input these samples simultaneously as a set to the flow $f_\theta$. To reduce the input dimension, we resort to a set-embedding $\frac{1}{L}\sum_{l=1}^L g(o_m^l)$, where $g$ is a neuralized nonlinear function to be learned. Depending on the structure of the model, one define the embedding as a Deepset~\citep{ZahKotRavPocetal17}, or, if the posterior is not invariant with respect to the order of the observations (e.g., hidden Markov models), one need to use a set-embedding that is not order invariant. To conclude, for a mini-batch Bayesian update, the parameterization of flow velocity can be modified as 
\begin{align*}
    f=\vh\left( \textstyle{\frac{1}{N}\sum_{n=1}^N} \phi(\vx_m^n),\textstyle{\frac{1}{L}\sum_{l=1}^L} g(o_{m+1}^l),\vx(t),t \right ).
\end{align*}
{\bf Sequence-segmentation.} We will use the approach in \eqref{eq:seq_break} to break a long sequence into short ones. More precisely, suppose we have particles $\{\vx_{m^*}^n\}_{n=1}^N$ at $m^*$-th stage. We can cut the sequence at position $m^*$ and generate a new task using the second half. The prior for the new sequence will be an empirical density estimation $\hat{\pi}(\vx;\gX_{m^*})$ as defined in~\eqref{eq:kde_prior}. Then, for all stages $m > m^*$, the terms in~\eqref{eq:joint-density} becomes
\vspace{-1mm}
\begin{align*}
  \log p(\vx, \gO_m) \approx \log \hat{\pi}(\vx_m^n;\gX_{m^*}) + {\sum_{t={m^*+1}}^m}\log p(o_t|\vx_m^n).
\end{align*}
We can apply this technique for multiple times to split a long sequence into multiple segments.
\begin{figure*}[ht!]
    \vspace{-1mm}
    \centering
    \begin{tabular}{@{}c@{}c@{}}
         \includegraphics[width=0.49\linewidth]{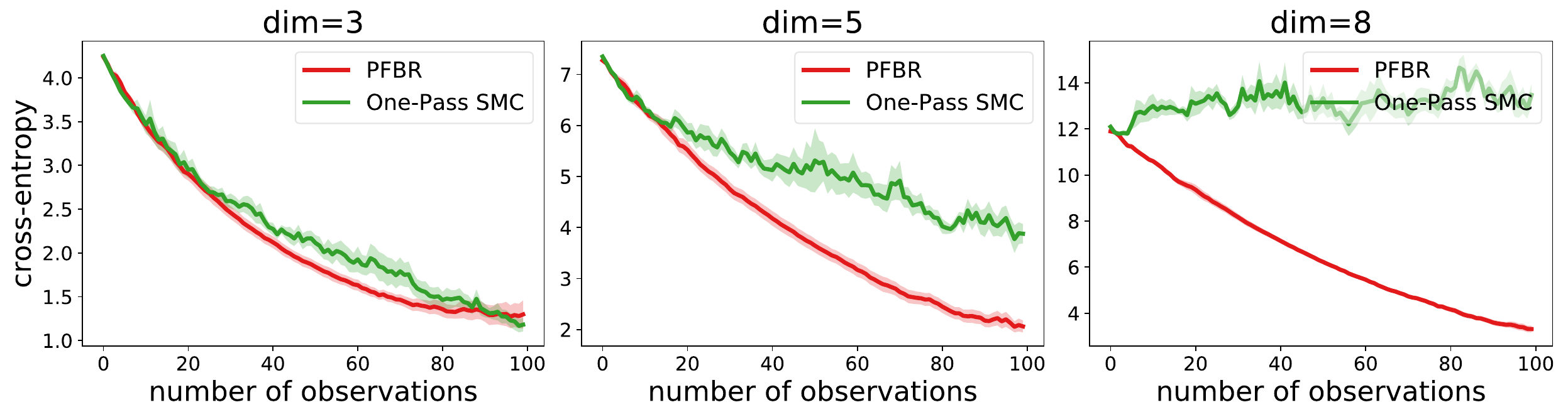} 
         &
          \includegraphics[width=0.51\linewidth]{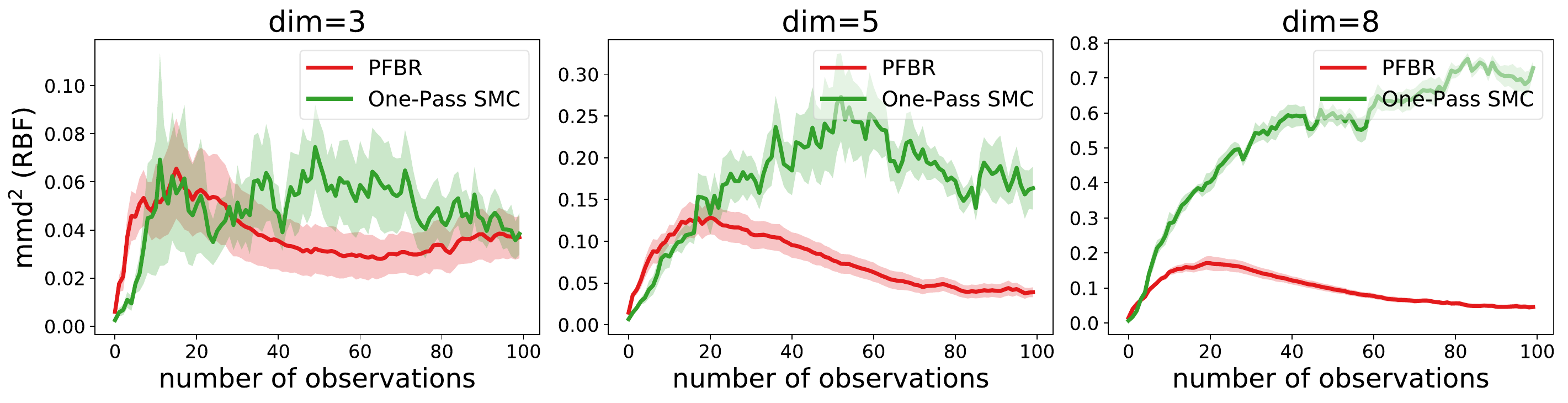}\\[-1mm]
          {\small (a) Cross entropy $\E_{p(\vx|\gO_m)} -\log q_m$ for dimension 3, 5 and 8} & {\small (b) Squared MMD with RBF kernel for dimension 3, 5 and 8}
    \end{tabular}
    \vspace{-3mm}
    \caption{\small Experimental results for Gaussian model. We use 256 obtained particles as samples from $p(\vx|\gO_m)$ and compare it with true posteriors. Each evaluation result is computed over 25 tasks and the shaded area shows the standard error. More results in Appendix~\ref{apx:experiments-results}.}
    \label{fig:mvn-exp}
    \vspace{-3mm}
\end{figure*}

\section{Experiments}
\label{sec:experiments}
\vspace{-0.5mm}

We conduct experiments on multivariate Gaussian model, hidden Markov model and Bayesian logistic regression to demonstrate the generalization ability of PFBR and also its accuracy for posterior estimation.

{\bf Evaluation metric.} For multivariate Gaussian model and Gaussian linear dynamical system, we could calculate the true posterior. Therefore, we can evaluate:
\vspace{-2mm}
\begin{itemize}[nosep, nolistsep, wide]
    \item[(i)] Cross-entropy $\E_{x\sim p}-\log q(x)$;
    \item[(ii)] Maximum mean discrepancy~\citep{GreBorRasSchetal12}
    \begin{align*}
      \text{MMD}^2= \E_{x,x'\sim p;~y,y'\sim q}[k(x,x')-2k(x,y)+k(y,y')];
    \end{align*}
    \item[(iii)] Integral evaluation discrepancy $\|\E_p h(x) - \E_q h(x)\|_2$;
\end{itemize}
where we use Monte Carlo method to compute $\E_{x\sim p}[\cdot]$ for the first two metrics. For the integral evaluation, we choose some test functions $h$ where the exact value of $\E_p h(x)$ has a closed-form expression. For the experiments on real-world dataset, we estimate the commonly used prediction accuracy due to the intractability of the posterior.

{\bf Multivariate Gaussian Model.}
The prior $\vx\sim \gN(\mu_{x},\Sigma_x)$, the observation conditioned on prior $o|\vx \sim \gN(\vx, \Sigma_o)$ and the posterior all follow Gaussian distributions. In our experiment, we use $\mu_x=\Vec{0}, \Sigma_x=I$ and $\Sigma_o=3I$. We test the learned PFBR on different sequences of 100 observations $\gO_{100}$, while the training set only contains sequences of 10 observations, which are ten times shorter than sequences in the testing set. However, since we construct a set of different prior distributions to train the operator, the diversity of the prior distributions allows the learned $\gF$ to generalize to compute posterior distributions in a longer sequence.

\vspace{-0.5mm}
We compare the performances with KBR~\citep{FukSonGre12} and one-pass SMC~\citep{BalMad06}. Both PFBR and KBR are learned from the training set and then used as algorithms on the testing set, which consists of 25 sequences of 100 observations $\{ \gO_{100}^j \sim \gN(\vx^j,\Sigma_o)\}_{j=1}^{25}$ conditioned on 25 different $\vx^j$ sampled from $\gN(\mathbf{0},I)$. 
We compare estimations of $p(\vx|\gO_m^j)$ across stages from $m=1$ to 100 and the results are plotted in Fig.~\ref{fig:mvn-exp}. Since KBR's performance reveals to be less comparable in this case, we leave its results to Appendix~\ref{apx:experiments-results}. We can see from Fig.~\ref{fig:mvn-exp} that as the dimension of the model increases, our PFBR has more advantages over one-pass SMC.

\begin{figure}
\vspace{-1mm}
\resizebox{1.0\columnwidth}{!}{%
\renewcommand{\arraystretch}{0} 
	\begin{tabular}{
	@{}c@{\hspace{0.1mm}}c@{\hspace{0.1mm}}c@{\hspace{0.1mm}}c@{\hspace{0.1mm}}c@{\hspace{0.1mm}}c@{\hspace{0.1mm}}
	c@{}
	}
\vspace{-1mm}
\rotatebox{90}{\tiny ~True~~}~~ &
\includegraphics[width=0.03\textwidth]{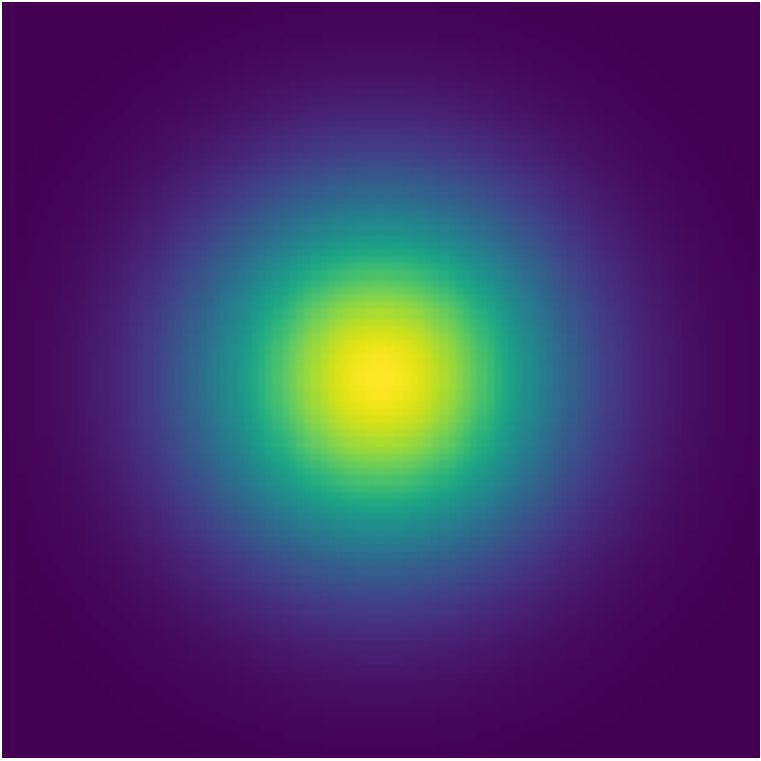} &
\includegraphics[width=0.03\textwidth]{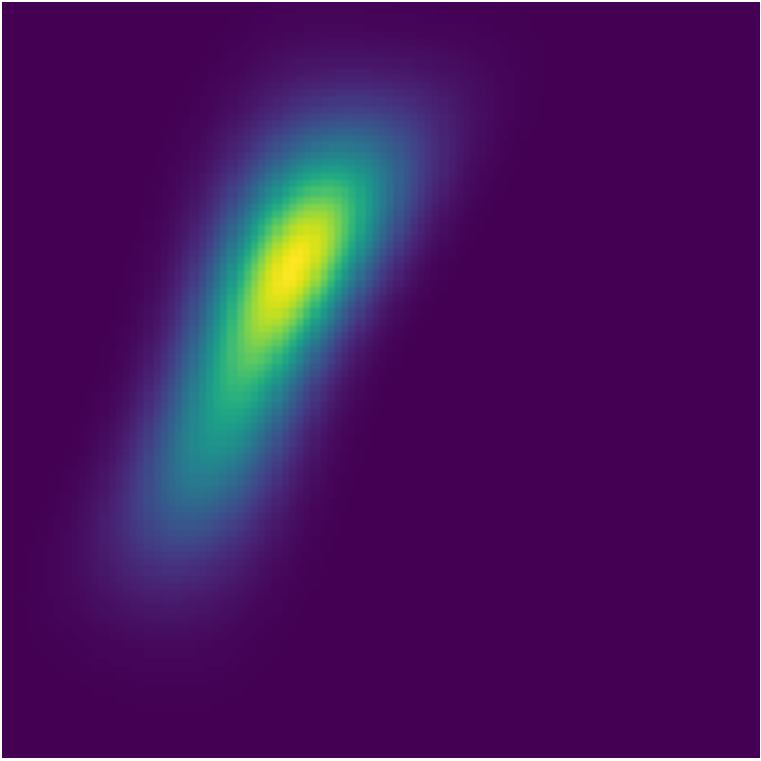} &
\includegraphics[width=0.03\textwidth]{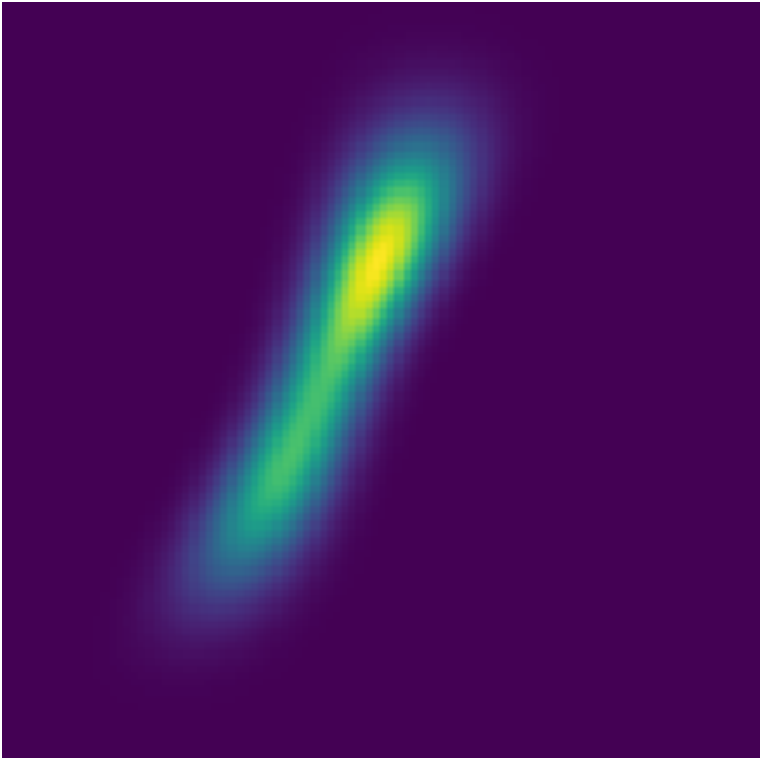} &
\includegraphics[width=0.03\textwidth]{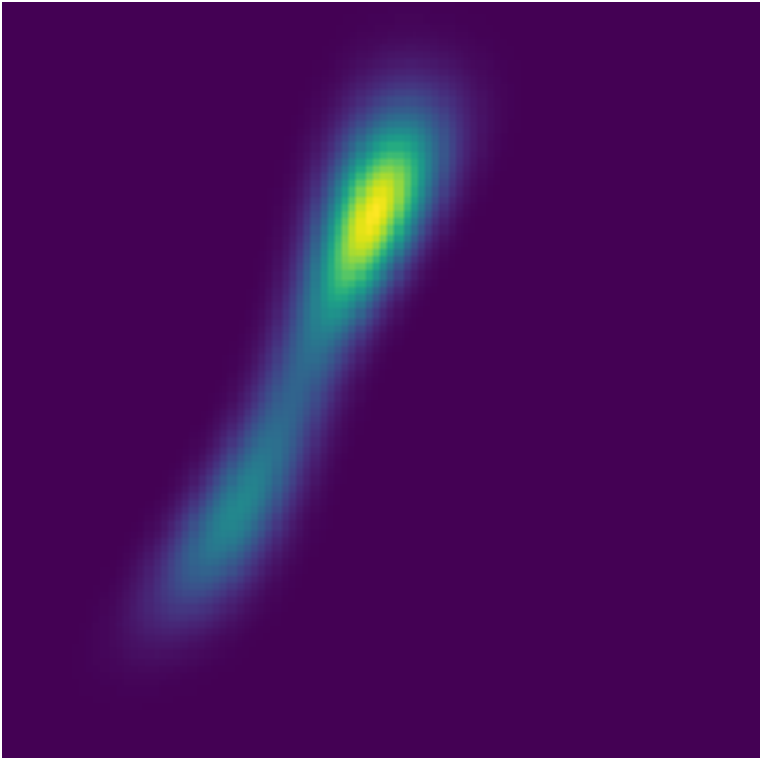} &
\includegraphics[width=0.03\textwidth]{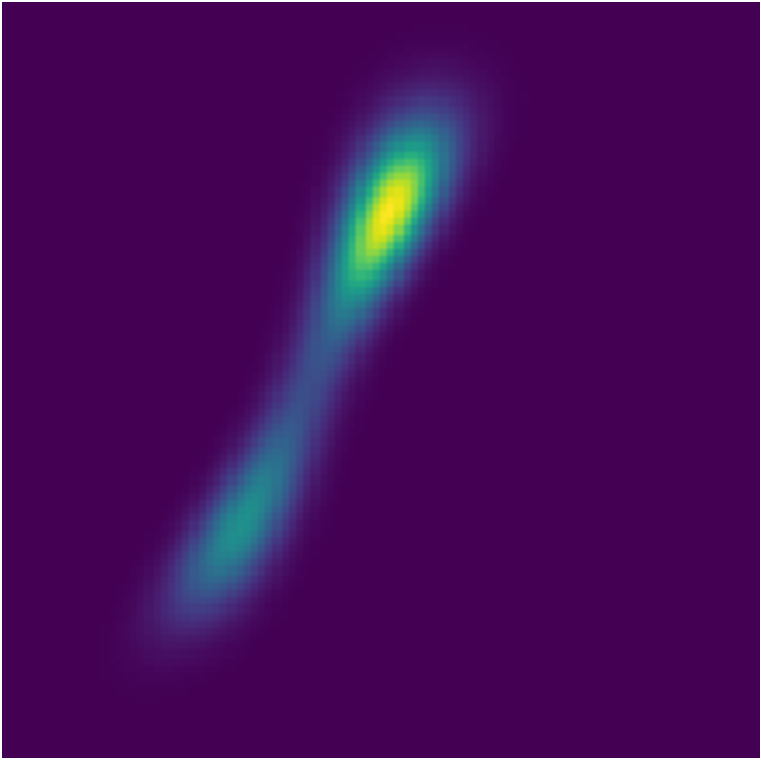} &
\includegraphics[width=0.03\textwidth]{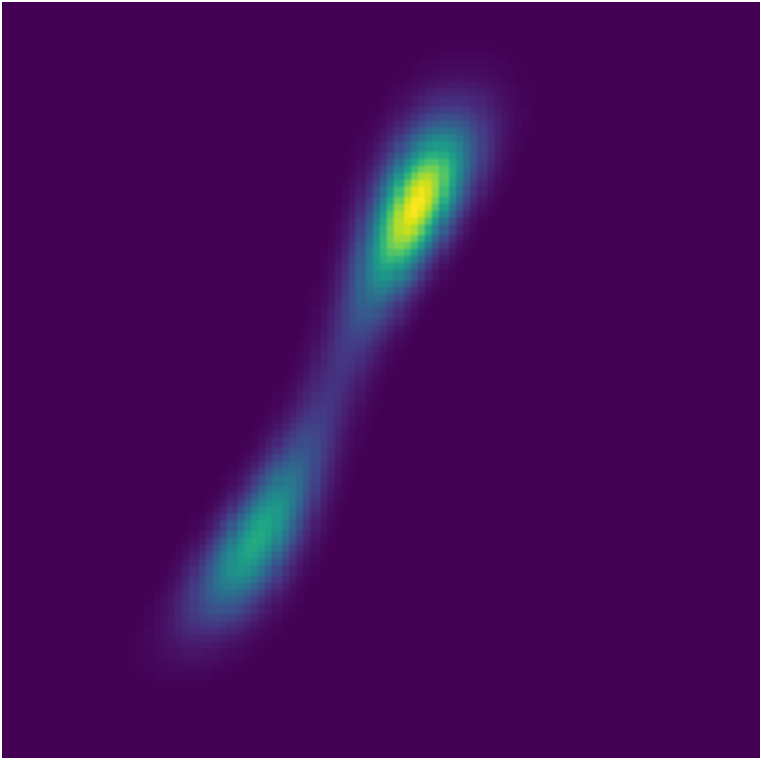} \\
\vspace{-1mm}
\rotatebox{90}{\tiny ~PFBR~~}~~ &
\includegraphics[width=0.03\textwidth]{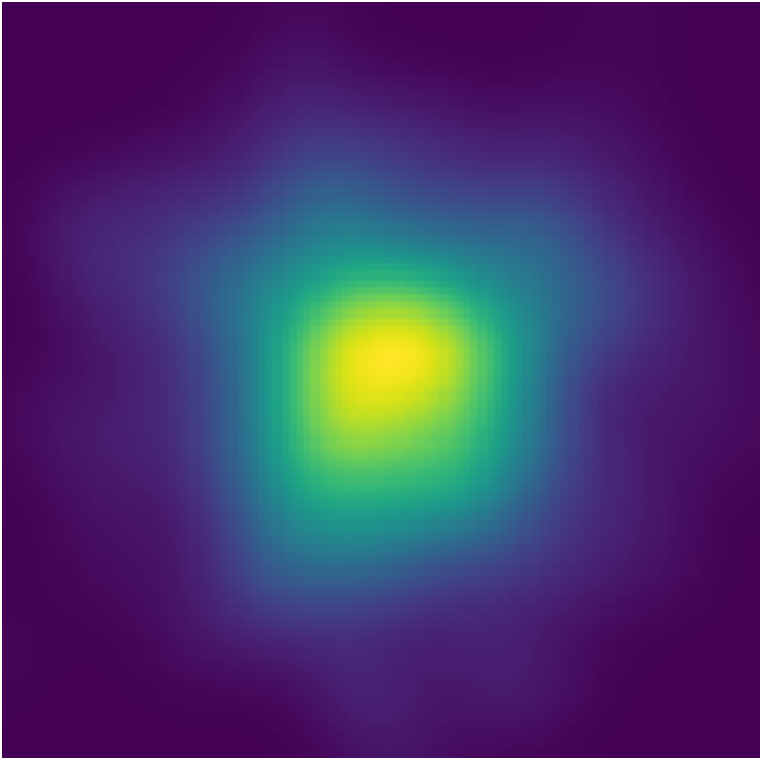} &
\includegraphics[width=0.03\textwidth]{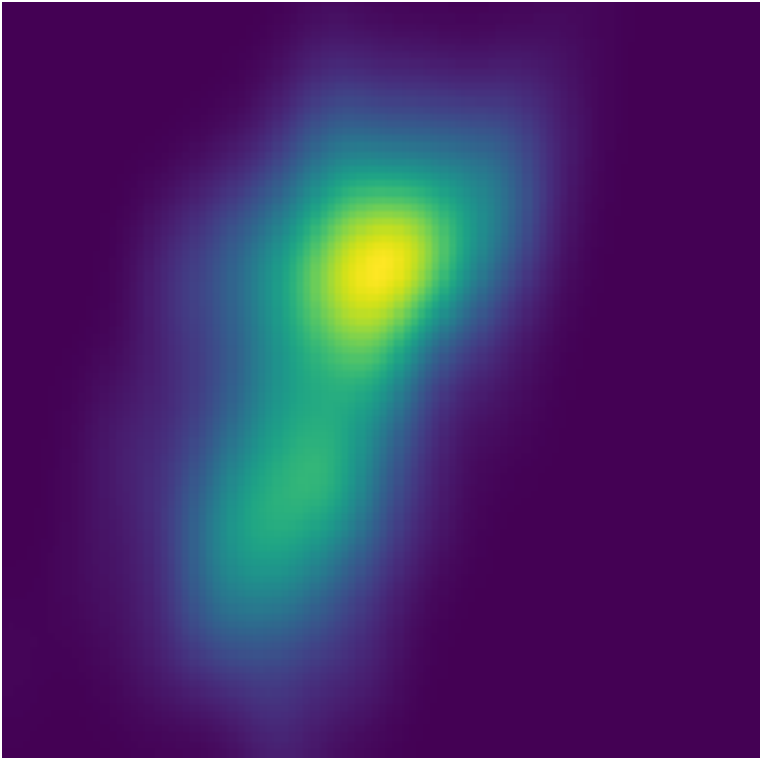} &
\includegraphics[width=0.03\textwidth]{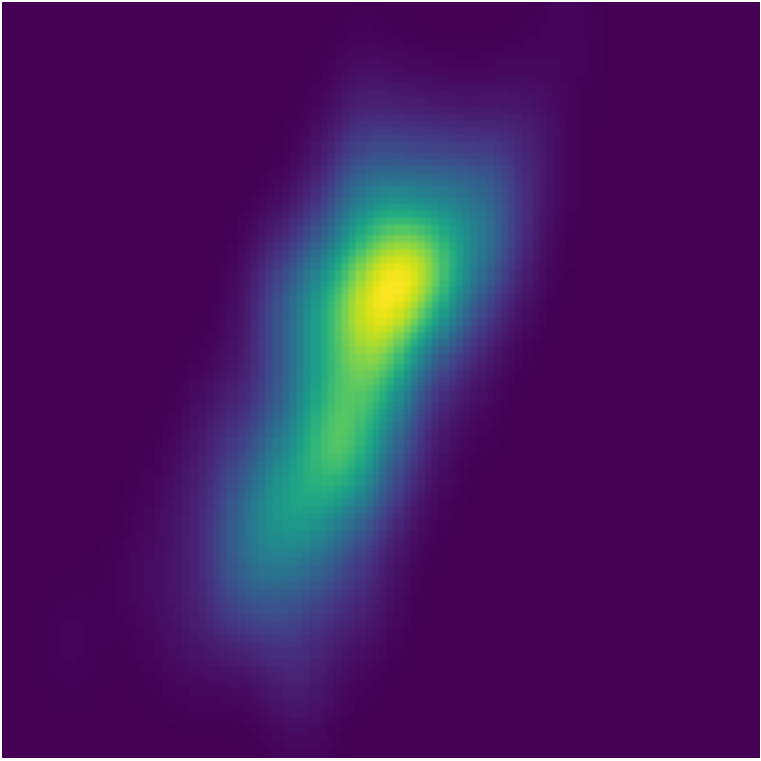} &
\includegraphics[width=0.03\textwidth]{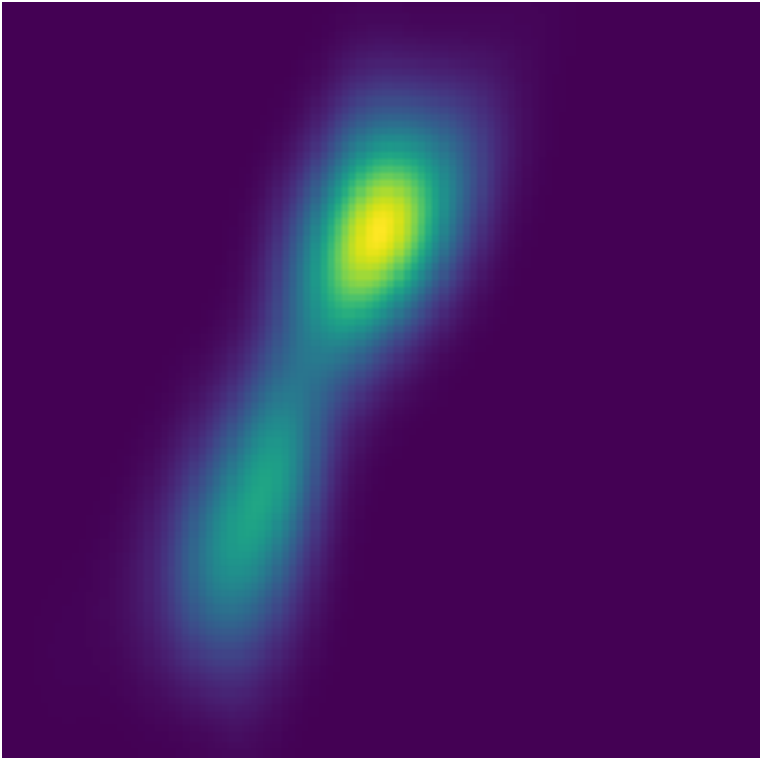} &
\includegraphics[width=0.03\textwidth]{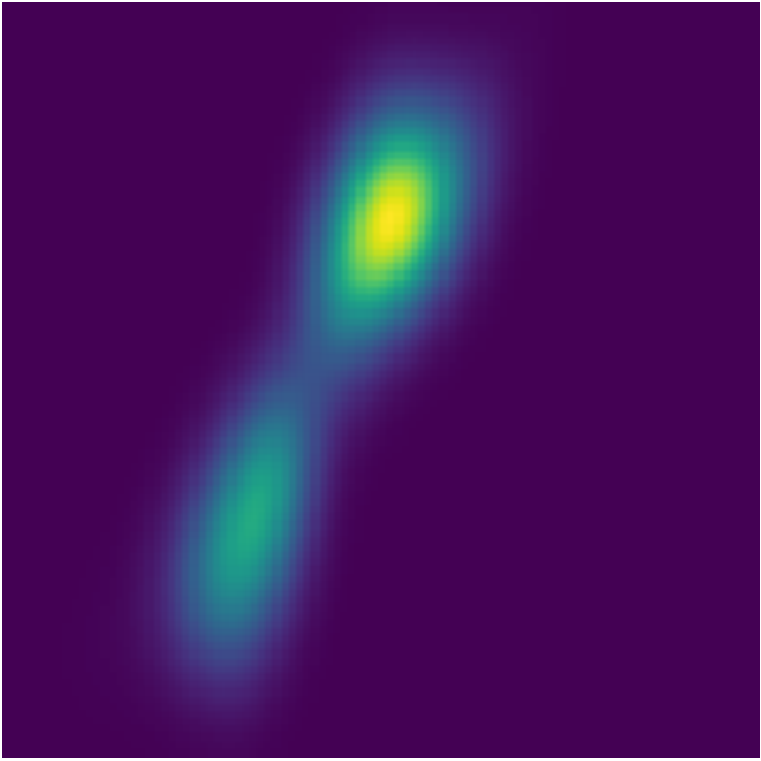} &
\includegraphics[width=0.03\textwidth]{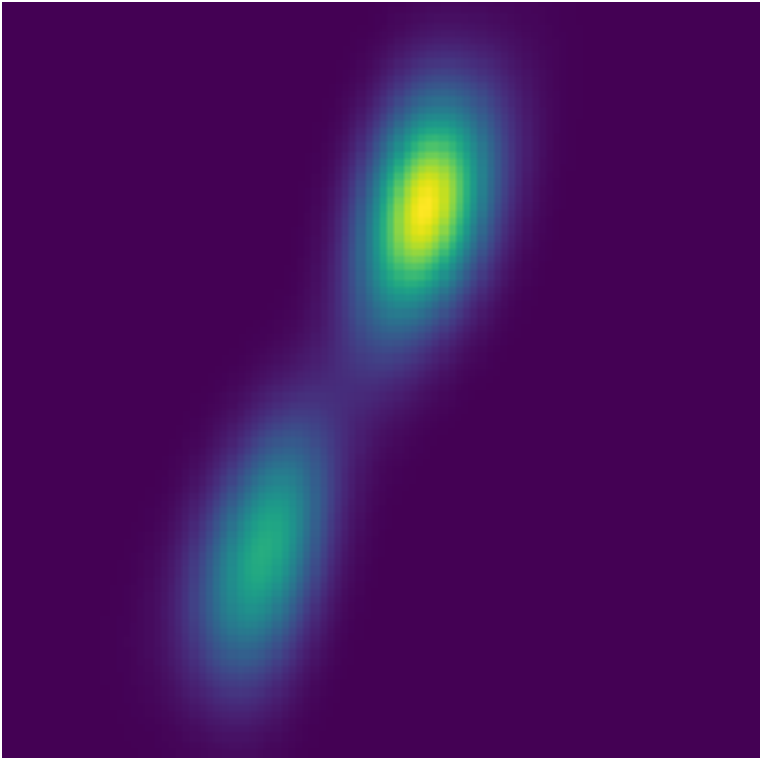} \\
\rotatebox{90}{\tiny ~SMC~~}~ &
\includegraphics[width=0.03\textwidth]{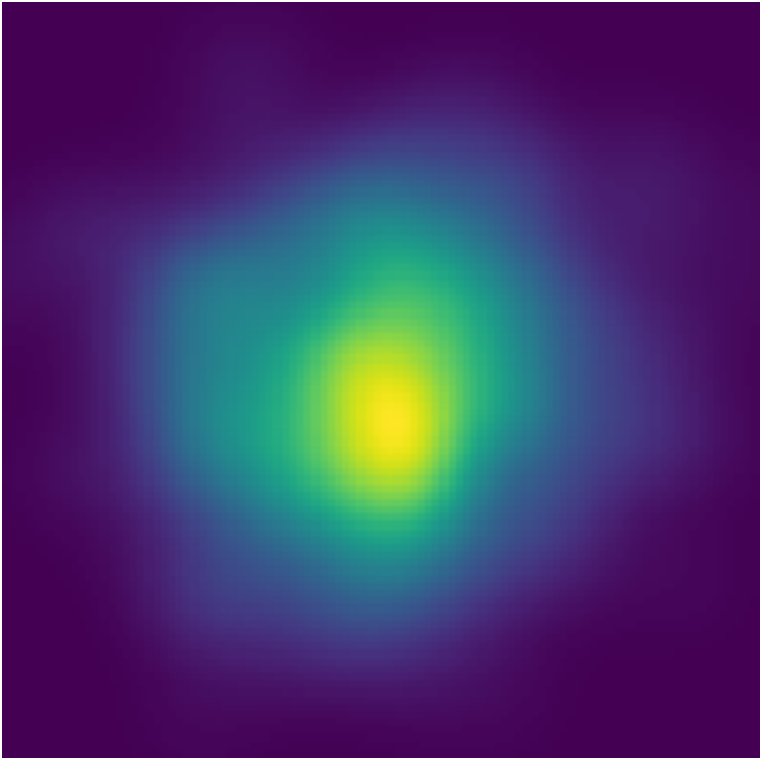} &
\includegraphics[width=0.03\textwidth]{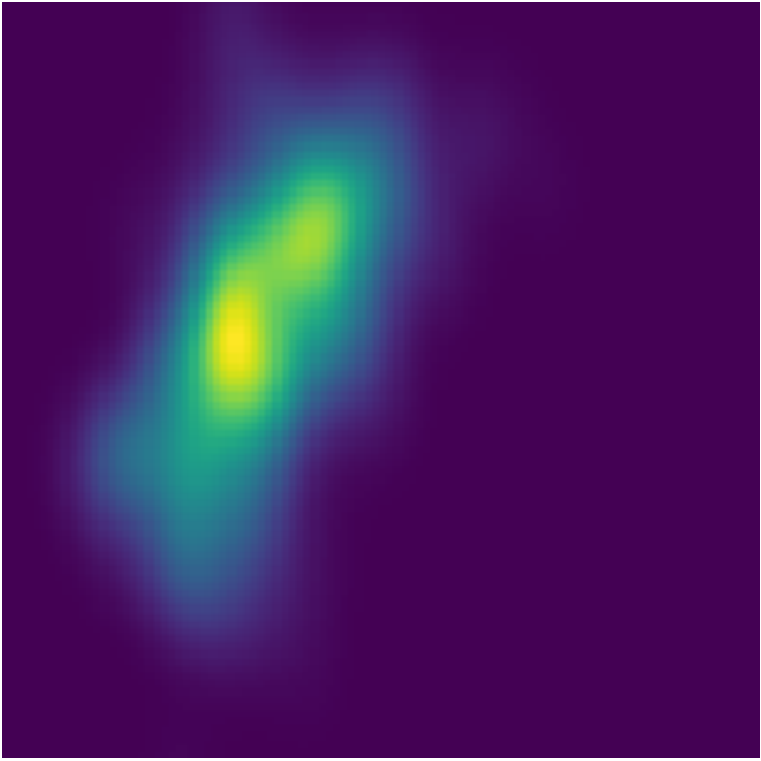} &
\includegraphics[width=0.03\textwidth]{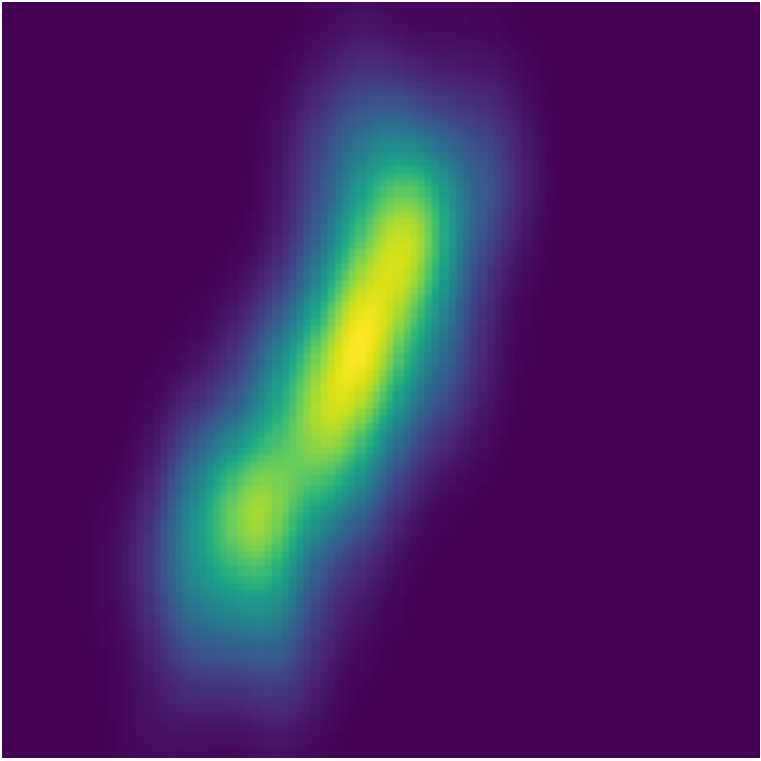} &
\includegraphics[width=0.03\textwidth]{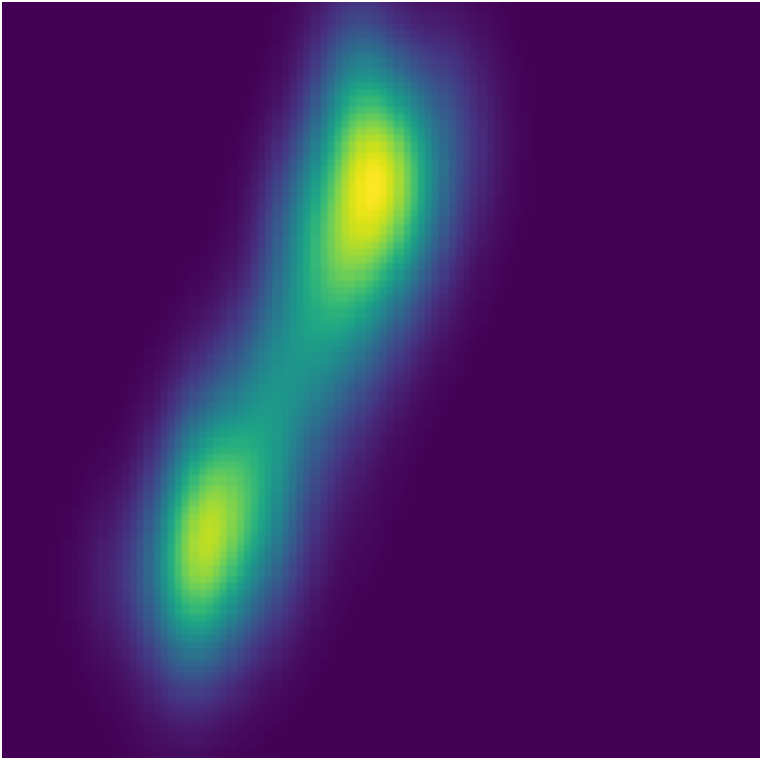} &
\includegraphics[width=0.03\textwidth]{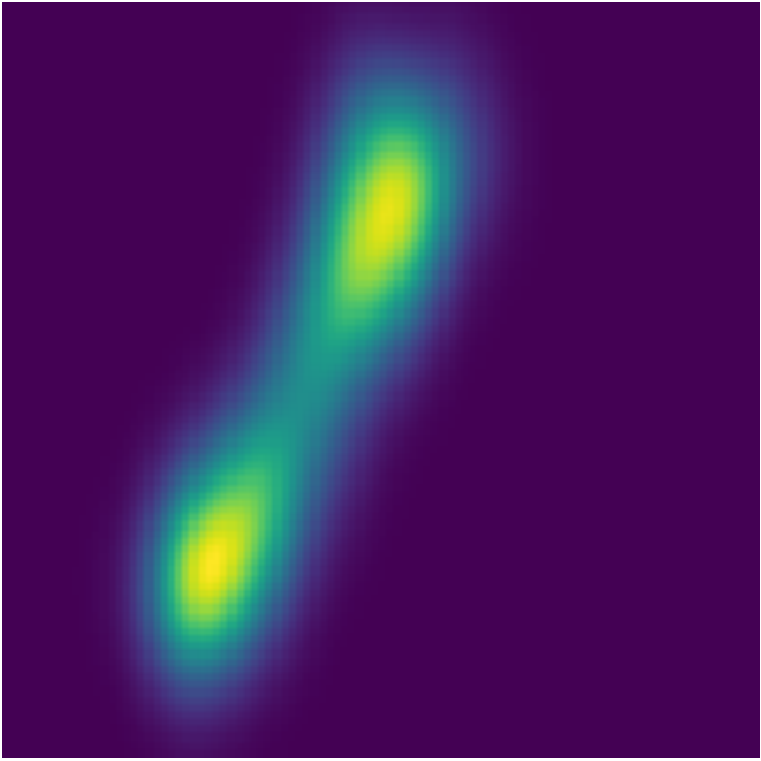} &
\includegraphics[width=0.03\textwidth]{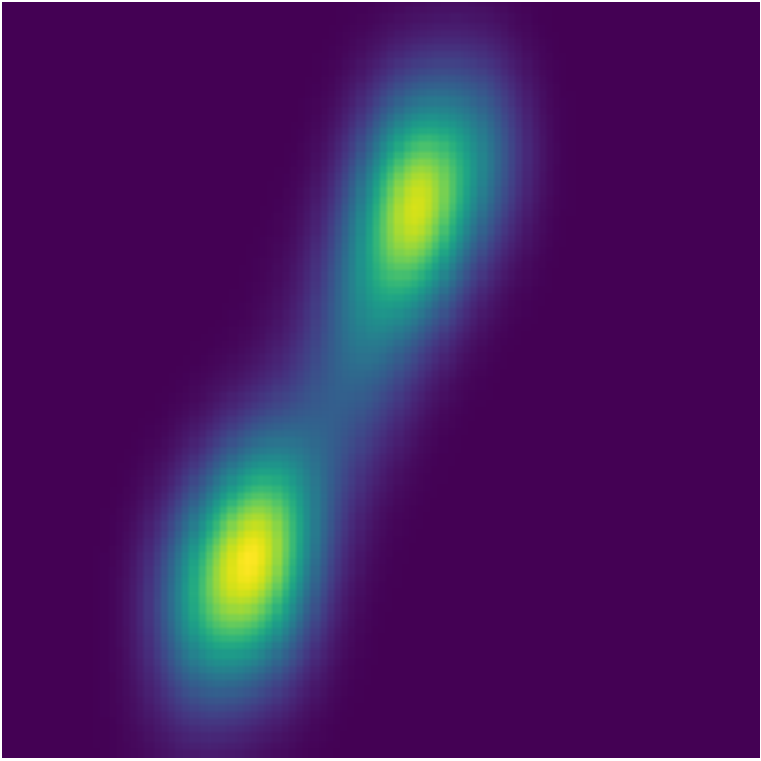} \\
	\end{tabular}
}
\vspace{-3mm}
	\caption{\small Visualization of the evolution of posterior density from left to right. In the end, PFBR density is closer to the true density than SMC density.}\label{fig:two-gaussian}
\vspace{-1mm}	
\end{figure}

{\bf Gaussian Mixture Model.}
Following the same setting as~\citet{WelTeh11}, we conduct an experiment on an interesting mixture model, where the observations 
$ o\sim \frac{1}{2}\big(\gN(x_1, 1.0) +\gN(x_1+x_2,1.0)\big)
$ and the prior $x_1, x_2 \sim \gN(0,1)$. The same as~\citet{DaiHeDaiSon16}, we set $(x_1,x_2) = (1,-2)$ so that the resulting posterior will have two modes: $(1,-2)$ and $(-1,2)$. During training, multiple sequences are generated, each of which consists of 30 observations.

\vspace{-0.5mm}
Compared to~\citet{WelTeh11} and~\citet{DaiHeDaiSon16}, our experiment is more challenging. First, they are only fitting one posterior given one sequence of observations via optimization, while PFBR will operate on sequences that are NOT observed during training without re-optimization. Second, they are only estimating the final posterior, while we aim at fitting every intermediate posteriors. Even for one fixed sequence, it is not easy to fit the posterior and capture both two modes, which can be seem from the results reported by~\citet{DaiHeDaiSon16}. However, our learned PFBR can operate on testing sequences and the resulting posteriors look closed enough to true posterior (See Fig.~\ref{fig:two-gaussian}).

{\bf Hidden Markov Model --- LDS.}
Our PFBR method can be easily adapted to hidden Markov models, where we will estimate the marginal posterior distribution $p(\vx_m|\gO_m)$. For instance, consider a Gaussian linear dynamical system (LDS):
\begin{align*}
    \vx_m = A \vx_{m-1} + \mathbf{\epsilon}_m, ~~~~  o_m  = B\vx_m + \mathbf{\delta}_m, 
\end{align*}
where $\mathbf{\epsilon}_m\sim \gN(0,\Sigma_1)$ and $\mathbf{\delta}_m\sim \gN(0,\Sigma_2)$.
The particles can be updated by recursively applying two steps:
\begin{align*}
    \text{(i)}~\tilde{\vx}_{m}^n = A {\vx}_{m-1}^n + \mathbf{\epsilon}_m,~~\text{(ii)}~{\vx}_{m}^n = \gF(\tilde{\gX}_m, \tilde{\vx}_m^n, o_{m+1} ),
\end{align*}
where $\tilde{\gX}_m:=\{\tilde{\vx}_m^n\}_{n=1}^N$.
The second step is a Bayesian update from  $p(\vx_m|\gO_{m-1})$ to $p(\vx_m|\gO_m)$ given the likelihood $p(o_m|\vx_m)$. The only tricky part is we do not have a tractable loss function in this case because of the integration $p(\vx_m|\gO_{m-1}) = \int p(\vx_{m}|\vx_{m-1})p(\vx_{m-1}|\gO_{m-1})\,d\vx_{m-1}$. Hence, at each stage $m$, we use the particles $\tilde{\gX}_m$ to construct an empirical density estimation $\hat{\pi}(\vx; \tilde{\gX}_m)$ as defined in~\eqref{eq:kde_prior} and then define the loss at each stage $m$ as
\[
    \textstyle{\sum_{n=1}^N} \log q_{m}(\vx_{m}^n) - \log p(o_m| \vx_m^n) - \log \hat{\pi}(\vx_m^n;\tilde{\gX}_m),
\]
where we replace the intractable density $p(\vx_m |\gO_{m-1})$ by $\hat{\pi}$. Given this loss, the PFBR operator $\gF$ can be learned.

In the experiment, we sample a pair of 2 dimensional random matrices $A$ and $B$ for the LDS model. We learn both PFBR and KBR from a training set containing multiple different sequences of observations. For KBR, we use an adapted version called Kernel Monte Carlo Filter~\citep{Kanagawa13}, which is designed for state-space models. We use Kalman filter~\citep{WelBis06} to compute the true marginal posterior for evaluation purpose.
\begin{figure}[h!]
\vspace{-2mm}
\begin{tabular}{cc}
    \includegraphics[width=0.44\linewidth]{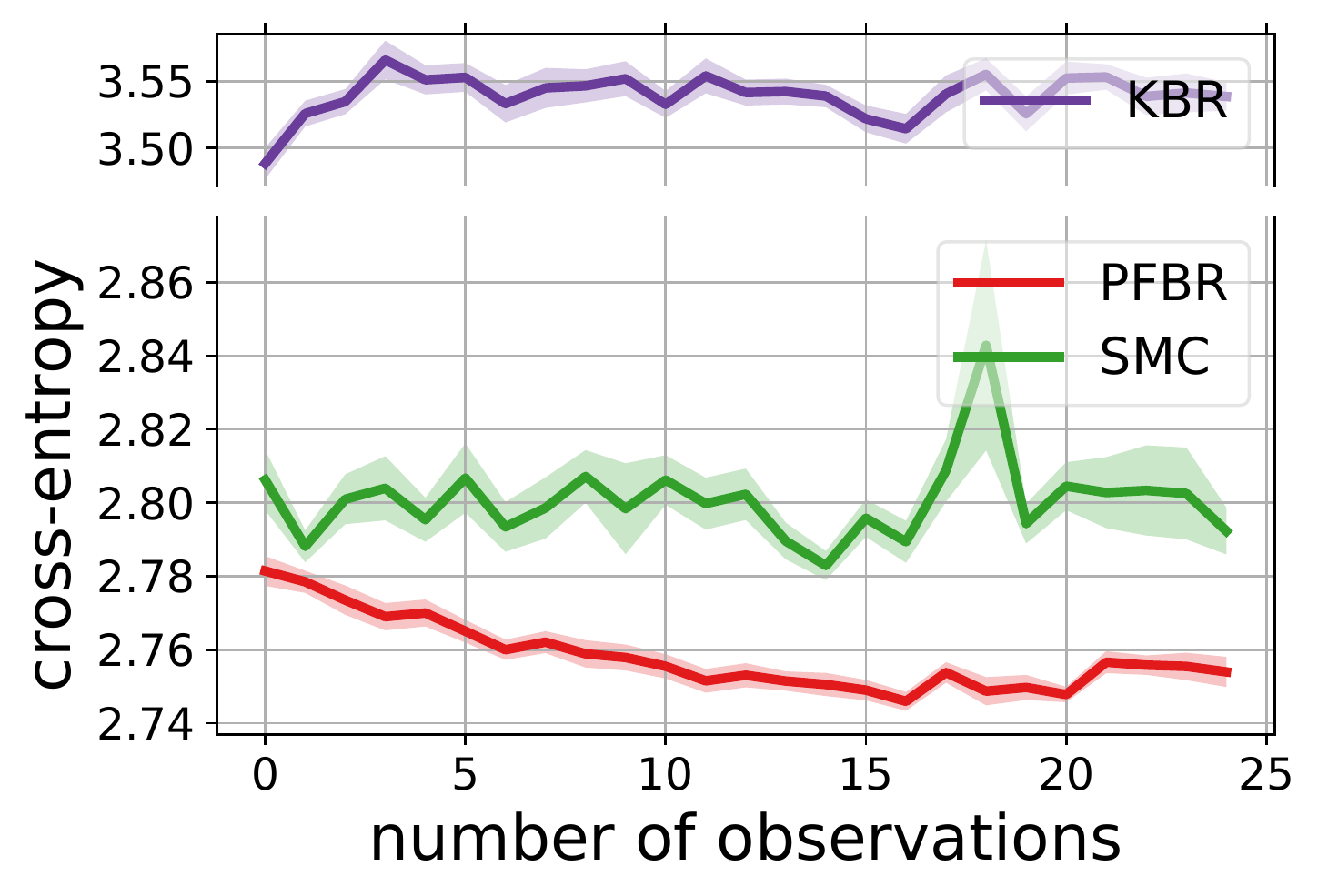} & \includegraphics[width=0.44\linewidth]{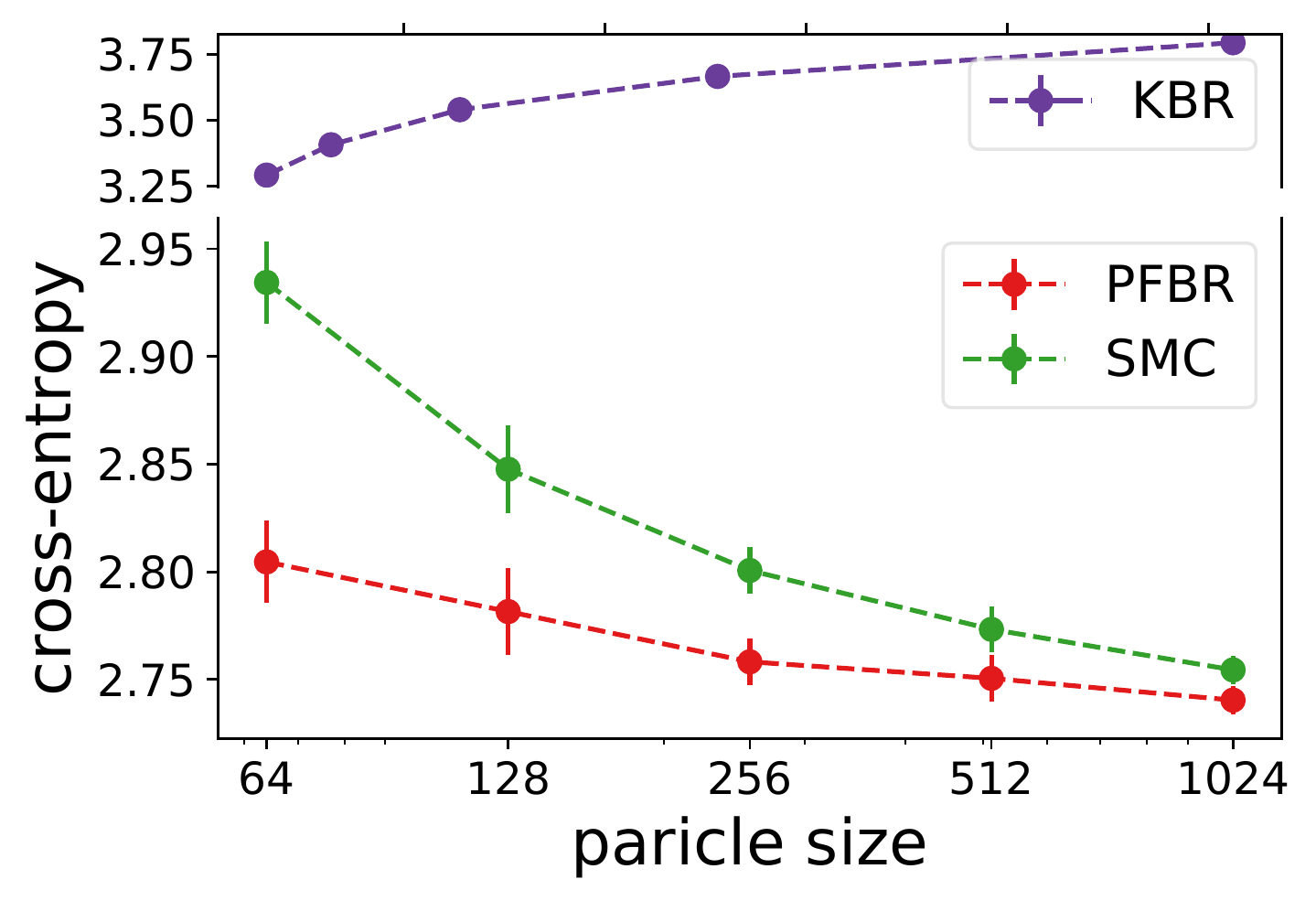} \vspace{-2mm}\\
    \multicolumn{2}{c}{
    {\scriptsize (a) Comparison of cross entropy }
    }\vspace{0.5mm}\\
     \includegraphics[width=0.44\linewidth]{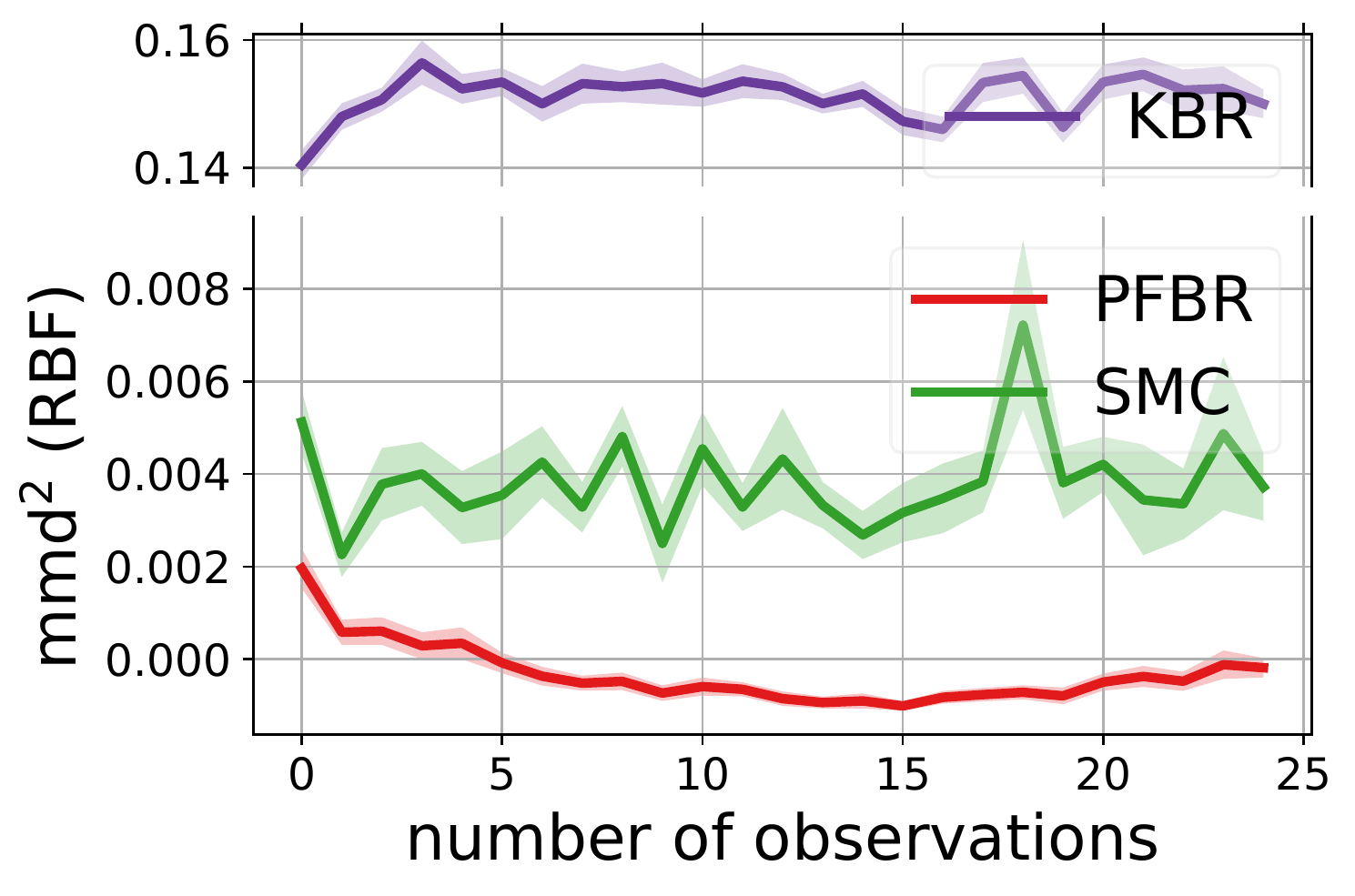} & \includegraphics[width=0.44\linewidth]{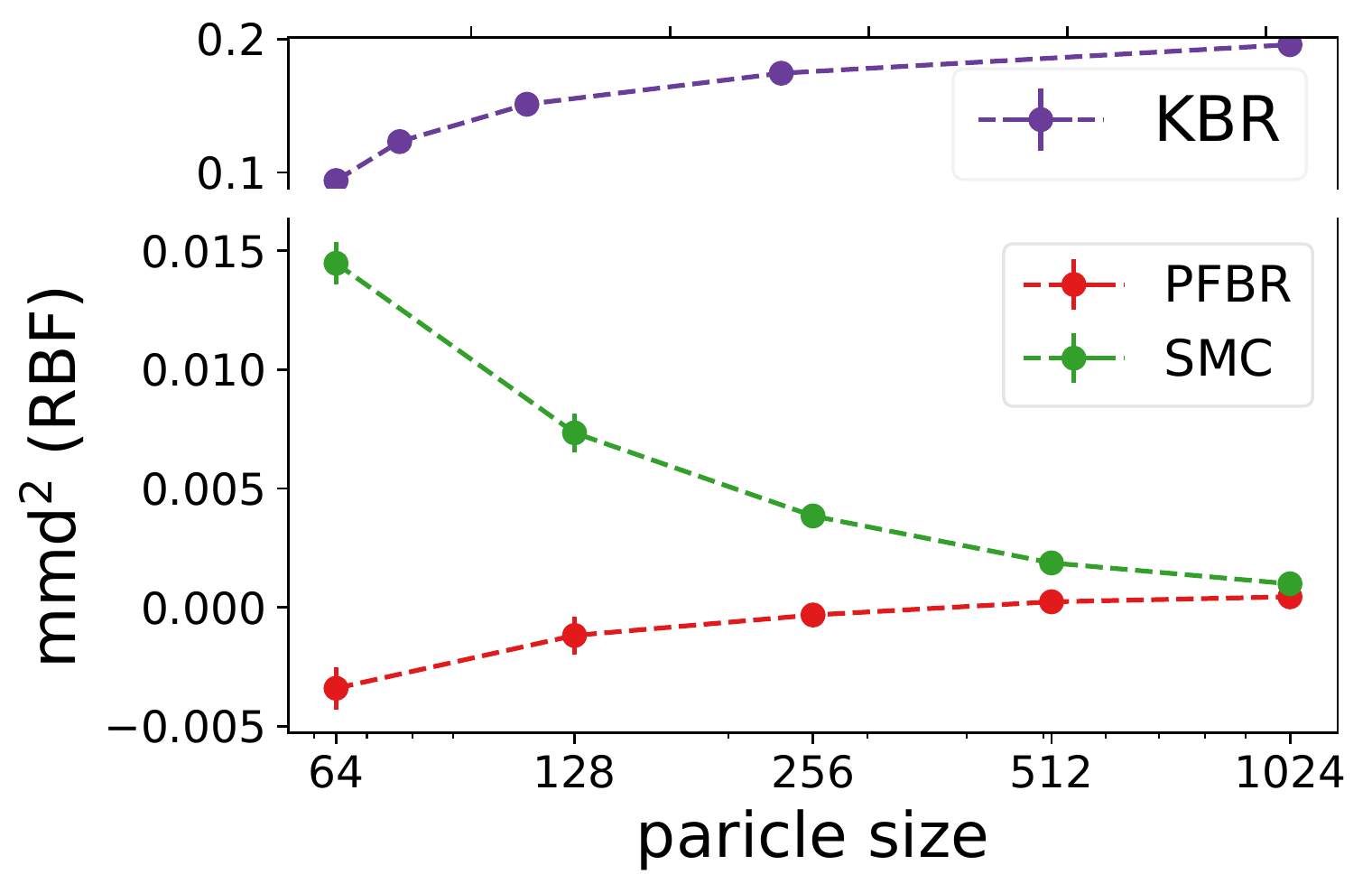}\vspace{-2mm}\\
     \multicolumn{2}{c}{
    {\scriptsize (b) Squared MMD with RBK kernel }
    }\vspace{0.5mm} \\
         \includegraphics[width=0.44\linewidth]{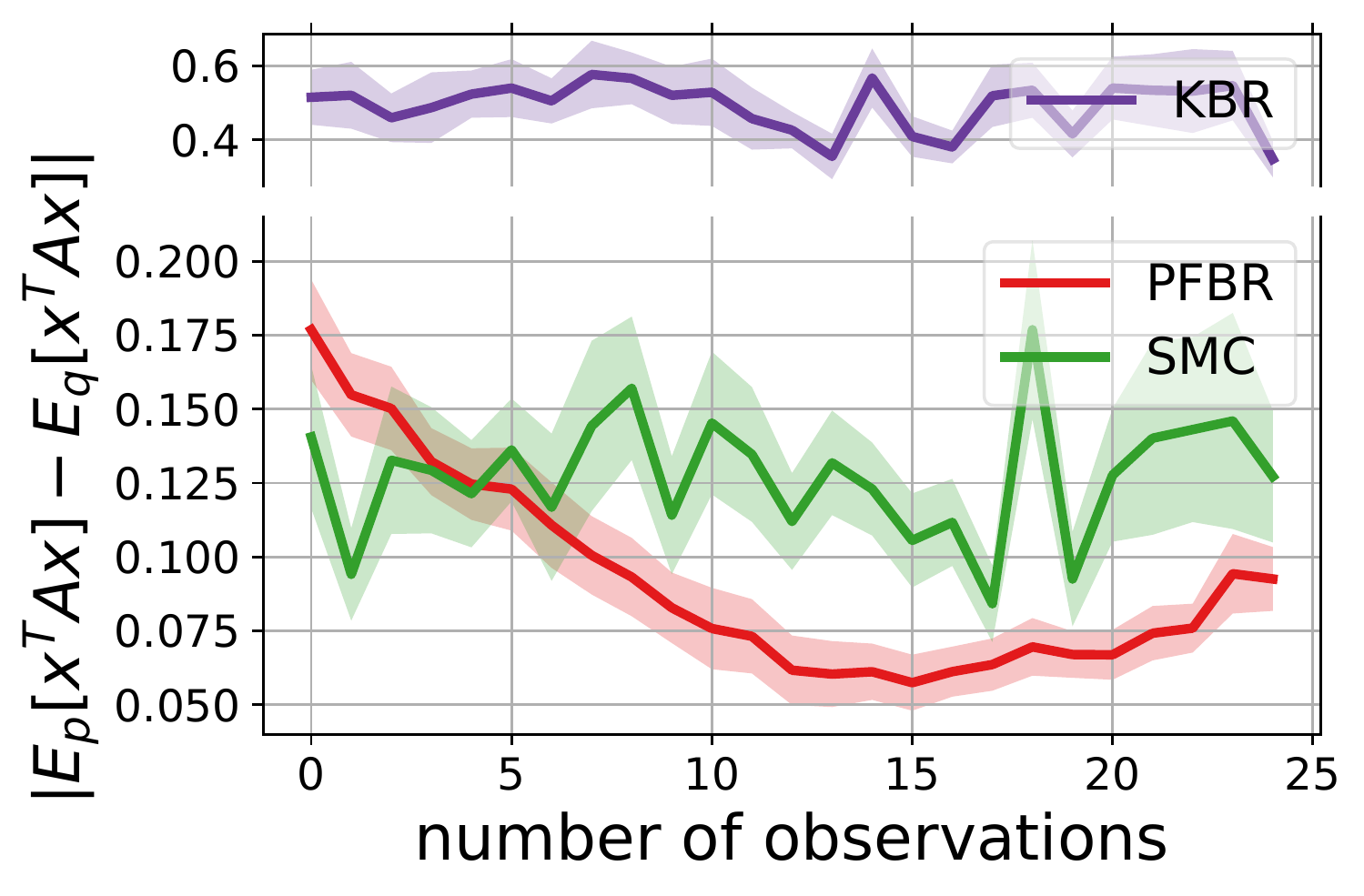} & \includegraphics[width=0.44\linewidth]{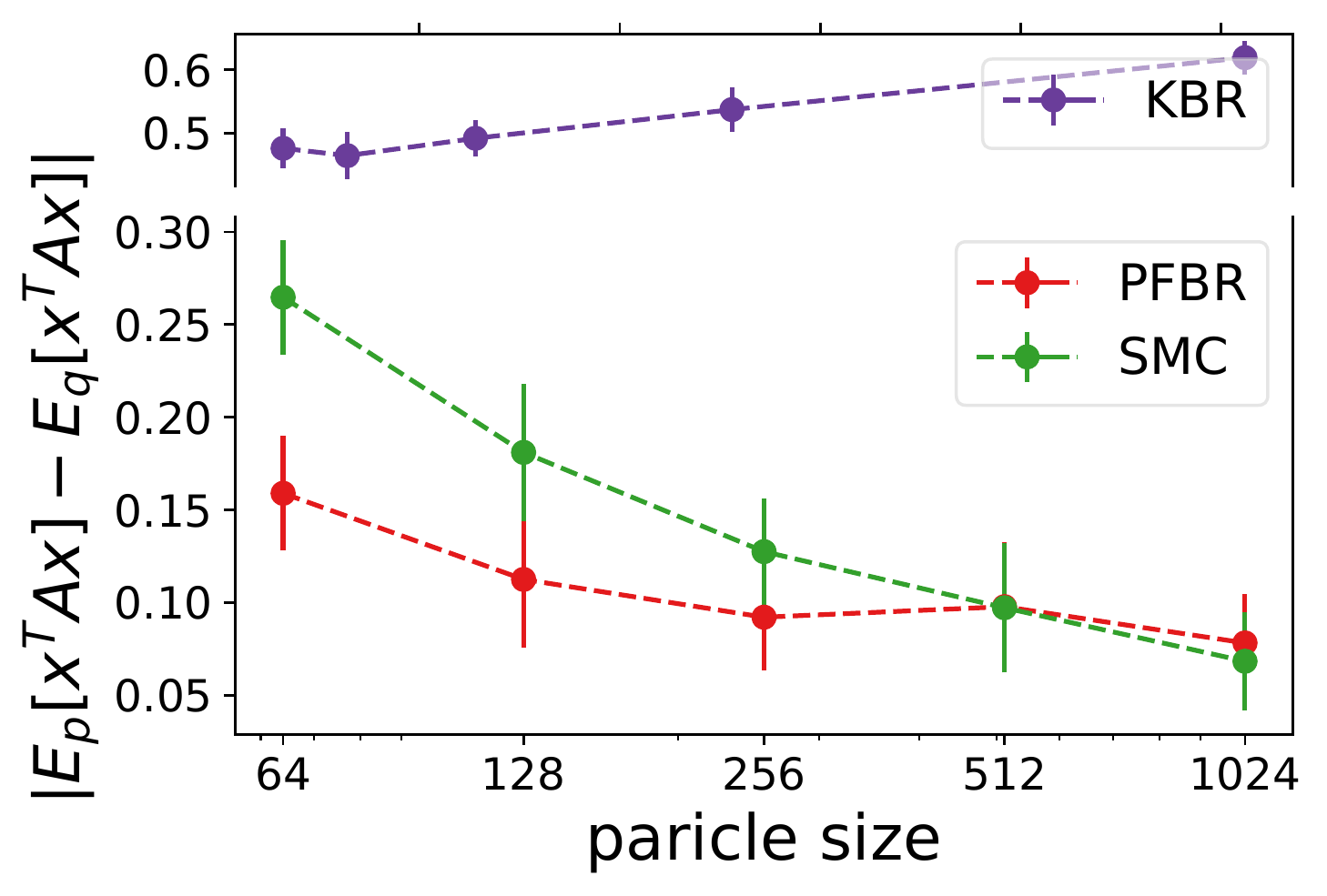}\vspace{-2mm}\\
     \multicolumn{2}{c}{
    {\scriptsize (c) Integral estimation on test function }
    }\vspace{-2mm}
\end{tabular}
    \vspace{-2mm}
    \caption{\small Experimental results for LDS. Only results evaluated on the testing set are shown. {\it Left:} estimation errors on every stage $m\in[25]$. {\it Right:} estimation errors for different particle sizes, which are firstly averaged over 25 stages for each task, and then averaged over 25 tasks. The error bar shows the standard error over tasks. We use the same PFBR operator trained with 1024 particles, even though during testing phase, we apply it on particles of difference sizes. See Appendix~\ref{apx:experiments-results} for more results.}
    \label{fig:exp-lds}
\end{figure}

Fig.~\ref{fig:exp-lds} compares our method with KBR and SMC (importance sampling with resampling) on a testing set containing 25 new sequences of ordered observations. We see that our learned PFBR can generalize to test sequences and achieve better and stabler performances.

{\bf Comparison to Variational SMC.} Autoencoding SMC {\small(AESMC)}, Filtering Variational Objectives, and Variational SMC are three recent variational inference approaches that approximate the posterior based on SMC and a neural proposal~\citep{anh2018autoencoding,maddison2017filtering,naesseth2018variational}. Since they share similar ideas, we implemented {\small AESMC} as a representative. We tried both {\small MLP} and {\small GRU} as mentioned in these papers. A comparison is made for 10-dimensional LDS (Table~\ref{tab:vsmc}), which shows PFBR is better even with much fewer particles.

\begin{table}[h!]
    \centering
    \small
    \setlength{\tabcolsep}{2pt}
    \begin{tabular}{|c|c|c|c|c|}
    \hline
    Algo & \#particles & cpu time (s) & gpu time (s) & cross-entropy  \\
    \hline
PFBR & 256 & 0.23 & 0.26 & 16.56 \\
SMC & 256 & 0.07 & 0.02 & 26.78 \\
ASMC-mlp & 256 & 0.17 & 0.07 & 19.66 \\
ASMC-gru & 256 & 0.18 & 0.07 & 19.38 \\
ASMC-mlp & 4096 & 2.23 & 0.25 & 17.63 \\
ASMC-gru & 4096 & 2.26 & 0.26 & 17.24 \\
SMC & 8192 & 3.87 & 0.12 & 17.60 \\
\hline
    \end{tabular}
    \vspace{-1mm}
    \caption{\small Numbers are averaged over 25 sequences with 25 observations each. For PFBR, gpu is faster when \#particles is larger.}
    \label{tab:vsmc}
\end{table}

{\bf Inference Time Comparison.} Table~\ref{tab:vsmc} also shows the inference time for updating posterior given one new observation. Though PFBR takes more time for the same \#particles (e.g., 256), to get closer to PFBR's performance, others need to use much more particles (e.g., 4096).

{\bf Bayesian Logistic Regression (BLR).}
We consider logistic regression for digits classification on the MNIST8M 8 vs. 6 dataset which contains about 1.6M training samples and 1932 testing samples. We reduce the dimension of the images to 50 by PCA, following the same setting as~\citet{DaiHeDaiSon16}. Two experiments are conducted on this dataset. For both experiments, we compare our method with SMC, SVI (stochastic variational inference~\citep{HofBleWanPai12}), PMD (particle mirror descent~\citep{DaiHeDaiSon16}), SGD Langevin (stochastic gradient Langevin dynamics~\citep{WelTeh11}) and SGD NPV (stochastic version of nonparametric variational inference~\citep{GerHofBle12}). This is a large dataset, so we use the techniques discussed in Section~\ref{sec:efficient-training} to facilitate training efficiency. 

{\bf BLR-Meta Learning.}
In the first experiment, we create a multi-task environment by rotating the first and second components of the features reduced by PCA through an angle $\psi$ uniformly sampled from $[-15\degree, 15\degree]$. Note that the first two components account for more variability in the data.
With a different rotation angle $\psi$, the classification boundary will change and thus a different classification task will be created. Also, a different sequence of image samples will result in different posterior distributions and thus corresponds to a different inference task.

We learn the operator $\gF$ from a set of training tasks, where each task corresponds to a different rotation angle $\psi$. After that, we use $\gF$ as Bayes' Rule for testing tasks and compare its performances with other stochastic methods or sampling methods. Test is done in an online fashion: all algorithms start with a set of particles sampled from the prior (hence the prediction accuracy at 0-th step is around 0.5). Each algorithm will make a prediction to the encountered batch of 32 images, and then observe their true labels. After that, each algorithm will update the particles and make a prediction to the next batch of images. Ideally we should compare the estimation of posteriors. However, since it is intractable, we evaluate the average prediction accuracy at each stage. Results are shown in Fig.~\ref{fig:blr-meta}.

\begin{figure}[h!]
    \vspace{-1mm}
    \begin{tabular}{cc}
    \includegraphics[width=0.455\linewidth]{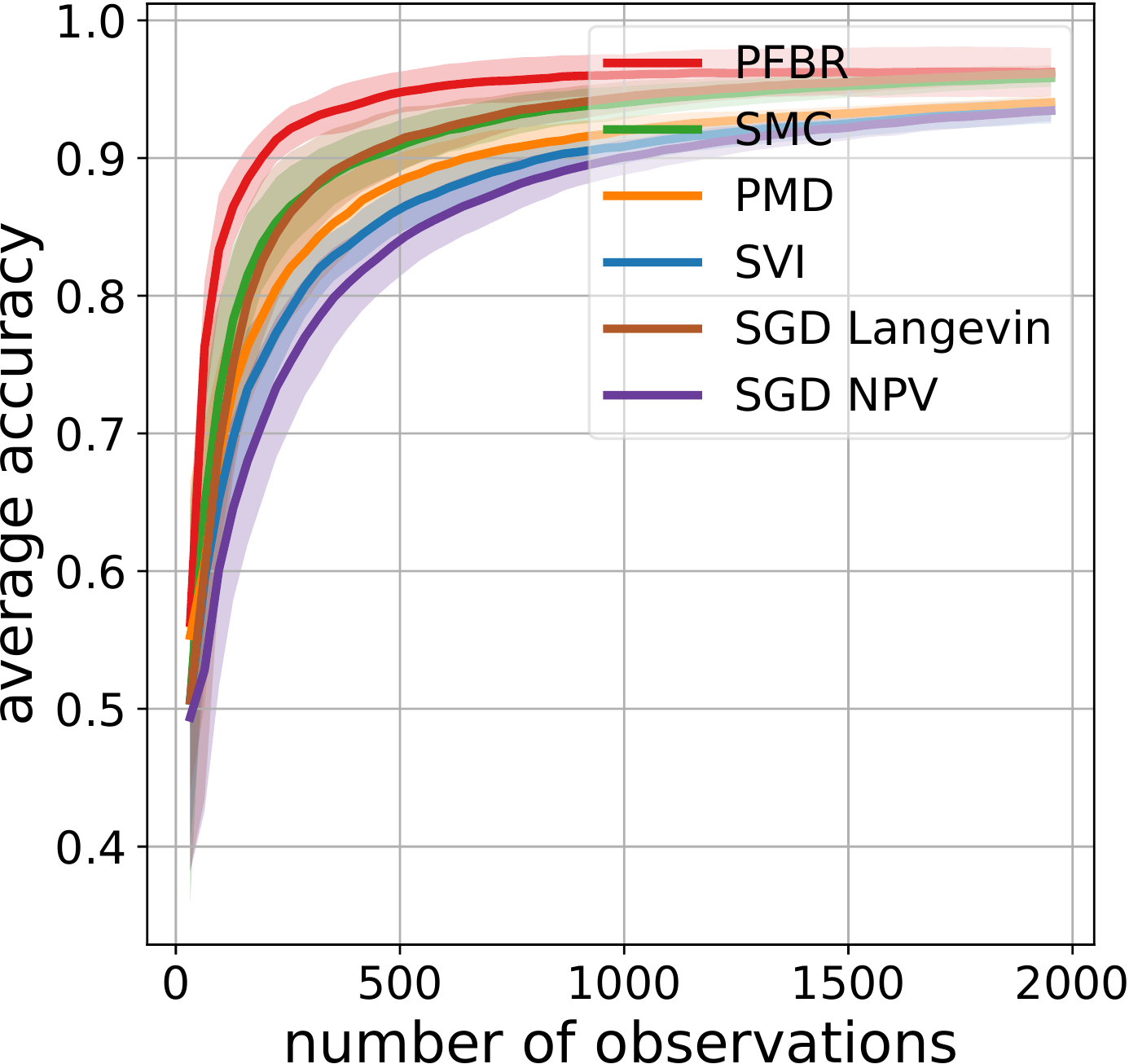}
    &
\includegraphics[width=0.45\linewidth]{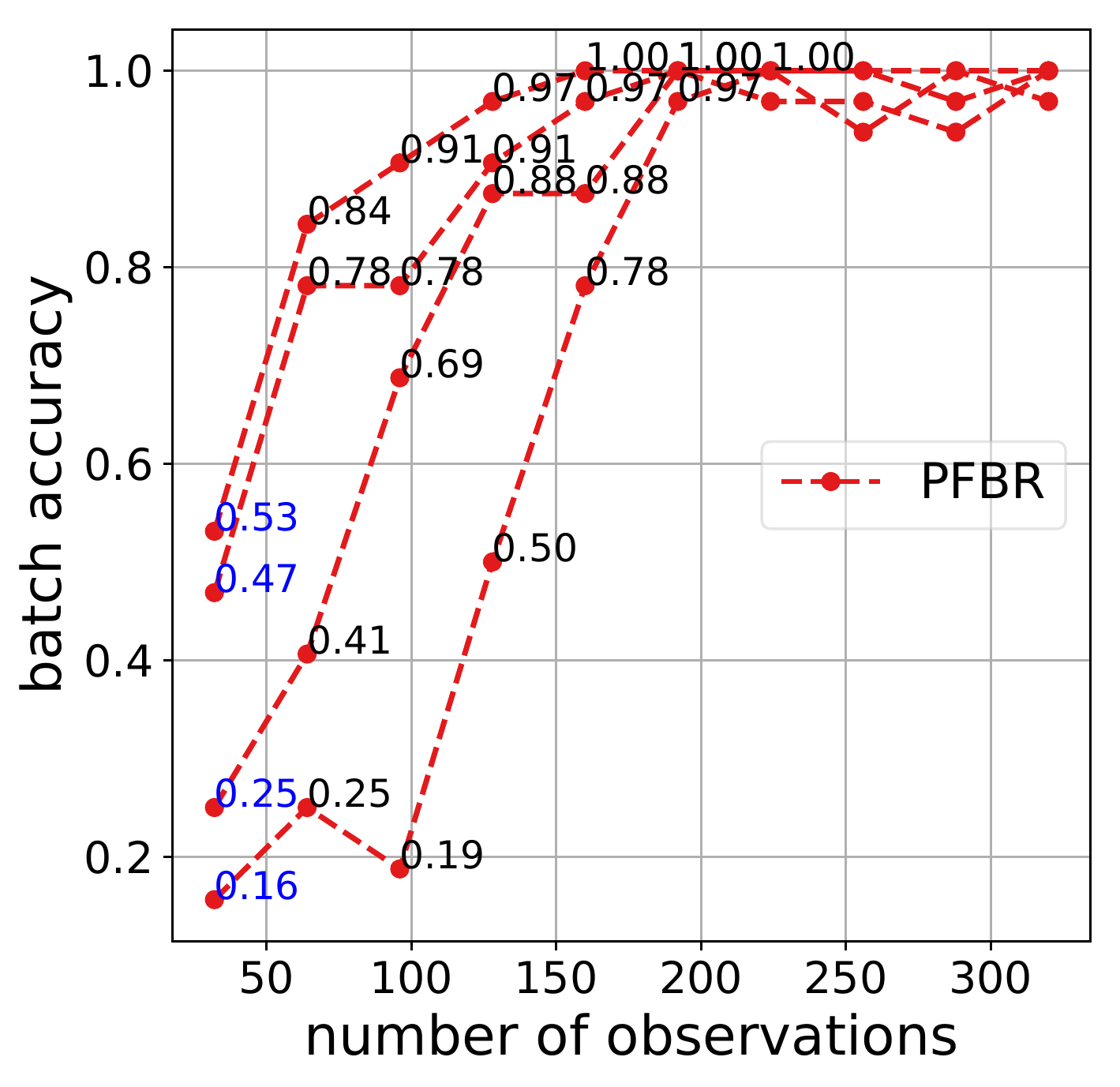}
    \end{tabular}
    \vspace{-2mm}
    \caption{\small Bayesian logistic regression on MNIST. {\it Left}: The average online prediction accuracy $\frac{1}{m}\sum_{t=1}^m r_t$ is evaluated, where $r_t$ is the accuracy for the $t$-th batch of images. The shaded area presents standard deviation of results over 10 testing tasks. {\it Right}: We collect some examples when the random initialization is farther from the posterior and gives worse initial prediction. PFBR $\gF$ updates those particles to gradually achieve a higher accuracy.}
    \label{fig:blr-meta}
\end{figure}

Note that we have conducted a sanity check to confirm the learned operator $\gF$ does not ignore the first 2 rotated dimensions and use the rest 48 components to make predictions. More precisely, if we zero out the first two components of the data and learn $\gF$ on them. The accuracy of the particles dropps to around 65\%. This further verifies that the learned PFBR indeed can generalize across different tasks. 

{\bf BLR-Variational Inference.}
For the second experiment on MNIST, we use PFBR as a variational inference method. That is to say, instead of accurately learning a generalizable Bayesian operator, the estimation of the posterior $p(\vx|\gO_{train})$ is of more importance. Thus here we do not use the loss function in~\eqref{eq:elbo} summing over all intermediate states, but emphasize more on the final error $\text{KL}(q(\vx)||p(\vx|\gO_{train}))$. After the training is finished, we only use the transported particles to perform classification on the test set but do not further use the operator $\gF$. The result is shown in Fig.~\ref{fig:lr-vi}, where $x$-axis shows number of visited samples during training. Since we use a batch size of 128 and consider 10 stages, the first gradient step of our method starts after around $10^3$ samples are visited.

\begin{figure}[h!]
\vspace{-2mm}
\centering
\includegraphics[width=0.55\linewidth]{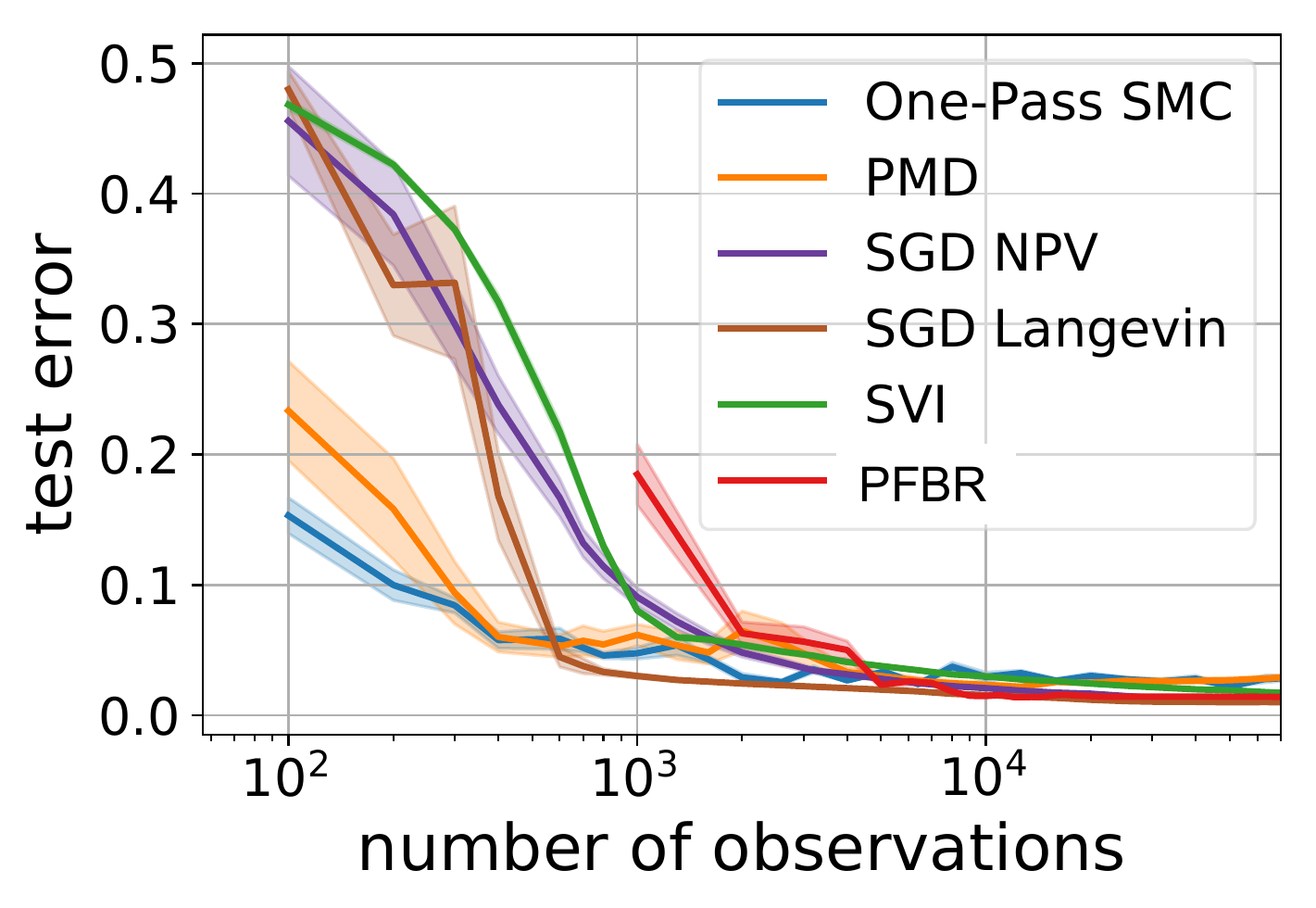}
\caption{\small PFBR as a variational inference method. The prediction accuracy of PFBR is comparable with state-of-art variational inference and sampling methods.}
\label{fig:lr-vi}
\end{figure}
\vspace{-3mm}
\section{Conclusion and Future Work}

In this paper, we have explored the possibility of learning an ODE-based Bayesian operator that can perform online Bayesian inference in testing phase and verified its generalization ability through both synthetic and real-world experiments. Further investigation on the parameterization of flow velocity $f$ (e.g., use a stable neural architecture~\citep{HaberRuthotto18} with HyperNetwork~\citep{HaDaiLe17}) and generating diverse prior distributions through a Dirichlet process can be made to explore a potentially better solution to this challenging problem. 
\newpage 

\section*{Acknowledgements}

{ 
This project was supported in part by NSF IIS-1218749, NIH
BIGDATA 1R01GM108341, NSF CAREER IIS-1350983,
NSF IIS-1639792 EAGER, NSF IIS-1841351 EAGER, NSF CNS-1704701, ONR
N00014-15-1-2340, Intel ISTC, NVIDIA, Google and Amazon
AWS. 
}

\bibliographystyle{icml2019}

\iftrue

\newpage
$ $
\appendix
\onecolumn

\section{Existence of Unified Flow Operator}\label{apx:existence}
\thmexist*

\begin{proof}
    By Theorem~\ref{thm:station-solution}, $\tilde{w}^*(q(\vx,t),t) := \nabla_x \log q(\vx,t)$ can induce the optimal state $\tilde{q}^*(\vx,\infty)=p(\vx|\gO_{m+1})$ and achieve a zero loss, $d = 0$. Hence, $\tilde{w}^*$ is an optimal closed-loop control for this problem.

Although in general closed-loop control has a stronger characterization to the solution, in a deterministic system like~\eqref{eq:opt-control2}, the optimal closed-loop control and the optimal open-loop control will give the same control law and thus the same optimality to the loss function\citep{Dreyfus64}. Hence, there exists an optimal open-loop control $w^
*=w^*(q(\vx,0),t)$ such that the induced optimal state also gives a zero loss and thus $q^*(\vx,\infty)=p(\vx|\gO_{m+1})$.

More specifically, when the system is deterministic, a state $q(\vx,t)$ is just a deterministic result of the initial state $q(\vx,0)$ and the dynamics. The optimal flow determined by $\tilde{w}^*(q(\vx,t),t)$ is 
\begin{align*}
f= \nabla_x\log p (\vx|\gO_{m})p (o_{m+1}|\vx )-\nabla_x \log q(\vx,t).
\end{align*}
The continuity equation gives
\begin{align*}
    \frac{\partial  q(\vx,t)}{\partial t}= &-\nabla_x\cdot\left( q\nabla_x\log p (\vx|\gO_{m})p (o_{m+1}|\vx)\right)\nonumber \\
    &+ \Delta_x q(\vx,t)\\
    &:=g(p(\vx|\gO_m)p(o_{m+1}|\vx), q(\vx,t))
\end{align*}

Hence, for any $\vx$,
\begin{align*}
    q(\vx,t) =&  q(\vx,0)+\int_0^t g(p(\vx|\gO_m)p(o_{m+1}|\vx), q(\vx,t))\,d\tau.\\
\end{align*}
The dynamcis $g$ is a fixed function of $p (\vx|\gO_{m})$, $p (o_{m+1}|\vx)$ and $q(\vx,t)$, so the solution of this initial value problem(IVP) $q(\vx,t)$ is a fixed function of $p (\vx|\gO_{m})$, $p (o_{m+1}|\vx)$, $q(\vx,0)$ and $t$, which can be written as
\begin{align*}
    q(\vx,t) =\text{Solve-IVP}(p (\vx|\gO_{m}),p (o_{m+1}|\vx),q(\vx,0),t).
\end{align*}
Finally, we can write the optimal open-loop control as
\begin{align*}
    &w^*(q(\vx,0),t) \\
    = &\nabla_x \log(\text{Solve-IVP}(p (\vx|\gO_{m}),p (o_{m+1}|\vx),q(\vx,0),t)).
\end{align*}
Hence, $w^*(q(\vx,0),t)$ has a fixed form across different $m$.  

\end{proof}

\section{Adjoint Method}\label{apx:adjoint-method}
To explain it more clearly, let us denote the evolution of the $n$-th particles at the $m$-th stage by $\vx^n_m(t)$ for $t\in[0,T]$. Note that $\vx_m^n(T)=\vx_{m+1}^n(0)$. (Then the notation $\vx_m^n$ in the main text will become $\vx_m^n(T)$.)

Recall the loss for each task:
 \begin{align*}
       \gL(\gT) =  \frac{1}{MN}{\sum_{m=1}^M} \sum_{n=1}^N \left(\log q_m^n(\vx_m^n(T),T)-\log p(\vx_m^n(T),\gO_m)\right).
 \end{align*}
The loss of one particle $\vx^n$ is \begin{align*}
    L^n:=&\frac{1}{M}\sum_{m=1}^M
    L_m^n,
\end{align*}
where 
\[
L_m^n:=-y_m^n(T) - \log p(\vx_m^n(T),\gO_m)
\]
and $y_m^n(t):=-\log q^n_m(x_m^n(t),t)$.

First, an adjoint process is defined as
\[\vp_m(t) :=\frac{\partial L^n}{\partial [\vx_m^n(t), y^n_m(t)]}.\] 
Denote $f_m(\vx(t),\theta) = f_{\theta}(\gX_m, o_{m+1}, \vx(t),t)$. During the $m$-th stage, the adjoint process follows the following differential equation
\begin{equation}\label{eq:backward-adjoint}
    \frac{d\vp_m}{dt} = -\frac{\partial}{\partial[\vx^n_m(t),y^n_m(t)]}\left[\begin{array}{c}
     f_m(\vx^n_m(t),\theta)  \\
     \nabla_{\vx}\cdot f_m(\vx^n_m(t),\theta) 
\end{array}
\right]^\top \vp_m(t).
\end{equation}
Note that
\begin{equation}
    \vp_m(T) =\sum_{m'\geq m}\frac{1}{M} \frac{\partial L^n_{m'}}{\partial [\vx_m^n(T),y_m^n(T)]}.
\end{equation}

Claim: The gradient of the loss is the solution of a backward ODE. That is to say, $\frac{\partial L^n}{\partial \theta} = \vz_1(0)$, if $ \vz_M(T) =\mathbf{0}$ and
\begin{align}
    \frac{d\vz_m(t)}{dt}= &- \left[\begin{array}{c}
         \frac{\partial f_m}{\partial \theta}(\vx^n_m(t),\theta)  \\
         \frac{\partial}{\partial \theta} \big[\nabla_{\vx}\cdot f_m(\vx^n_m(t),\theta)\big]
    \end{array}    
    \right]^\top
    \vp_m(t),\label{eq:backward-gradient1}
\end{align}
and $\vz_m(T)=\vz_{m+1}(0)$, for $m=0,\cdots, M-1$.

\begin{proof}
    First, we can compute $\frac{d}{dt}\frac{\partial L^n}{\partial \theta}$:
    \begin{align*}
       &\frac{d}{dt}\frac{ \partial L^n}{\partial \theta}= \frac{\partial}{\partial \theta}
      \sum_{m=1}^M \left( \frac{\partial L^n}{\partial \vx_m^n(t)}^\top\frac{d\vx_m^n(t)}{dt}+  \frac{\partial L^n}{\partial y^n_m(t)}\frac{dy_m^n(t)}{dt}\right)\\
         &= \frac{\partial}{\partial \theta}\sum_{m=1}^M\Big[ \vp_m(t)^\top \left[\begin{array}{c}
     f_m(\vx^n_m(t),\theta)  \\
     \nabla_{\vx}\cdot f_m(\vx^n_m(t),\theta) 
\end{array}
\right]\Big]\\
&=\sum_{m=1}^M \left[\begin{array}{c}
         \frac{\partial f_m}{\partial \theta}(\vx^n_m(t),\theta)  \\
         \frac{\partial}{\partial \theta} \big[\nabla_{\vx}\cdot f_m(\vx^n_m(t),\theta)\big]
    \end{array}    
    \right]^\top
    \vp_m(t)
    \end{align*}
    Next, we have 
    \[
   0- \frac{\partial L^n}{\partial \theta}=-\int_{t=0}^T
    \frac{d}{dt} \frac{\partial L^n}{\partial \theta}=\sum_{m=1}^M \int_{t=0}^T -\left[\begin{array}{c}
         \frac{\partial f_m}{\partial \theta}(\vx^n_m(t),\theta)  \\
         \frac{\partial}{\partial \theta} \big[\nabla_{\vx}\cdot f_m(\vx^n_m(t),\theta)\big]
    \end{array}    
    \right]^\top
    \vp_m(t) = \vz_M(T)-\vz_1(0).
    \]
    Hence, $\frac{\partial L^n}{\partial \theta}=\vz_1(0)$ if $\vz_M(T)=\mathbf{0}$.
\end{proof}
An algorithm for computing $\frac{\partial L}{\partial \theta}$ is summarized in Algorithm~\ref{alg_streamline_cfr}. A nice python package of realizing this algorithm is provided by~\citet{ChenRubanova18}.

\begin{algorithm}[h!]
  \DontPrintSemicolon
  \SetKwFunction{Grad}{Grad}
  \SetKwProg{Fn}{Function}{:}{}
  \SetKwFor{uFor}{For}{do}{}
  \SetKwFor{ForPar}{For all}{do in parallel}{}
  \SetKwFor{ForAll}{For all}{do}{}
\SetKwComment{Comment}{$\triangleright$\ }{}
  \SetCommentSty{mycommfont}
  
\Fn{\Grad{$\theta,\gX_0,p(o|\vx),\gO_M$}}{
Denote $f_{\theta}^m = f_{\theta}(\gX_m,o_{m+1},\vx(t),t)$\;\Comment*[r]{notation}

Set $y_0^n = -\log p(\vx_0^n)$ for each $\vx_0^n \in \gX_0$\;

\ForAll{$n=1$ to $N$}{
    \uFor{$m=0$ to $M-1$}{
    $\left[\begin{array}{c}
         \vx_{m+1}^n  \\
         y_{m+1}^n 
    \end{array}\right]\gets \left[\begin{array}{c}
         \vx_{m}^n  \\
         y_{m}^n 
    \end{array}\right] + \displaystyle{\int_0^T }\left[\begin{array}{c}
         f_{\theta}^m  \\
        \nabla\cdot f_{\theta}^m
    \end{array}\right]dt $
    }
    
    Set $\vp^n_M(T) = \mathbf{0}$ and $\vz^n_M(T)=\mathbf{0}$\;
    
    \uFor{$m=M$ to $1$}{
   $\vp^n_m(T) \gets \vp^n_m(T) + \frac{1}{M}\frac{\partial L_m^n}{\partial[\vx_m^n,y_m^n]}$\;
    
      Solve ODEs in~\eqref{eq:flow},~\eqref{eq:density_evolve},~\eqref{eq:backward-adjoint} and~\eqref{eq:backward-gradient1} for $\vx_m^n(t)$, $\vp^n_m(t)$ and $\vz^n_m(t)$ backwardly from $T$ to $0$\;
    
    Set $\vx_{m-1}^n(T)=\vx_m^n(0)$, $\vp_{m-1}^n(T)=\vp_m^n(0)$ and $\vz_{m-1}^n(T)=\vz_m^n(0)$ \;
    }
}

  \KwRet $ \frac{1}{N}\sum_{n=1}^N\frac{\partial L^n}{\partial \theta}= \frac{1}{N}\sum_{n=1}^N \vz^n_1(0) $   
}  

\caption{Adjoint Method of Computing the Gradient}
\label{alg_streamline_cfr}
\end{algorithm}
\newpage
\section{Experiment Details}\label{apx:experiments}
\subsection{Parameterization}\label{apx:parameterize}
Overall we parameterize the flow velocity as 
\begin{align*}
    f = \vh\left( \textstyle{\frac{1}{N}\sum_{n=1}^N} \bm{\phi}(\vx_m^n),o_{m+1},\vx(t),t \right ),
\end{align*}
where both $\bm{\phi}$ and $\vh$ are neural networks. For instance, let $\text{ctx}=[\textstyle{\frac{1}{N}\sum_{n=1}^N}\bm{\phi}(\vx_m^n)^\top,o_{m+1}^\top]$ be the context of this conditional flow, where $\bm{\phi}$ is a dense feed-forward neural network, a specific neural architecture we use in the experiment is
\begin{align}
 f=& \text{Gated}_k\left(\cdots[\text{ctx}, \text{Gated}_2\left([\text{ctx}, \text{Gated}_1\left([\text{ctx},\vx(t)^\top]^\top,t\right)]^\top,t\right)]^\top\cdots ,t \right),\\
    &\text{where}~\text{Gated}_j(\vy,t)=(W_j\vy+\vb_j)*\sigma(t \vv_j + \vc_j) + t\vc_j,
\end{align}
where $*$ is element-wise multiplication. The number of layers $k$ can be tuned, but in general $\vh$ is a shallow network.

\subsection{Evaluation Metric}
\begin{paragraph}{MMD$^2$}
The maximum mean discrepancy (MMD) of the true posterior $p$ and the estimated posterior $q$ is defined as
\begin{align*}
    \text{MMD}[\gF,p,q] := \sup_{f\in\gF}(\E_{x\sim p}[ f(x)] - \E_{y\sim q}[ f(y)]).
\end{align*}
When $\gF$ is a unit ball in a characteristic RKHS, ~\citet{GreBorRasSchetal12} showed that the squared MMD is
\begin{align*}
    \text{MMD}^2[\gF,p,q]=\E[k(x,x')]-2\E[k(x,y)]+\E[k(y,y')],
\end{align*}
 where $x, x'\sim p$ and $y,y'\sim q$.
\end{paragraph}

\begin{paragraph}{Cross-entropy}
Evaluating the KL divergence is equivalent to evaluating the cross-entropy.
\begin{align}
    \E_{x\sim p} -\log q(x) \approx \frac{1}{n}\sum_{n=1}^N(-\log q(x^n)),
\end{align}
where $q(x)$ is approximated by kernel density estimation on the set of particles obtained from different sampling methods.
\end{paragraph}
\begin{paragraph}{Integral Evaluation}
When the true posterior is a Gaussian distribution $\gN(\mu, \Sigma)$, the expectation of the following test functions have closed-form expressions.
\begin{itemize}[nosep, nolistsep, wide]
    \item $\E[\vx] = \mu$
    \item $\E[\vx^\top A \vx] = tr(A\Sigma) + \mu^\top A\mu$
    \item $\E[(A\vx+\va)^\top(B\vx+\vb)] = tr(A\Sigma B^\top)+(A\mu+\va)^\top(B\mu + \vb)$
\end{itemize}
\end{paragraph}
\newpage
\section{More Experimental Results}\label{apx:experiments-results}
\subsection{Multivariate Guassian Model}
\begin{figure}[h!]
\centering
 \begin{tabular}{ccc}
       \includegraphics[width=0.23\textwidth]{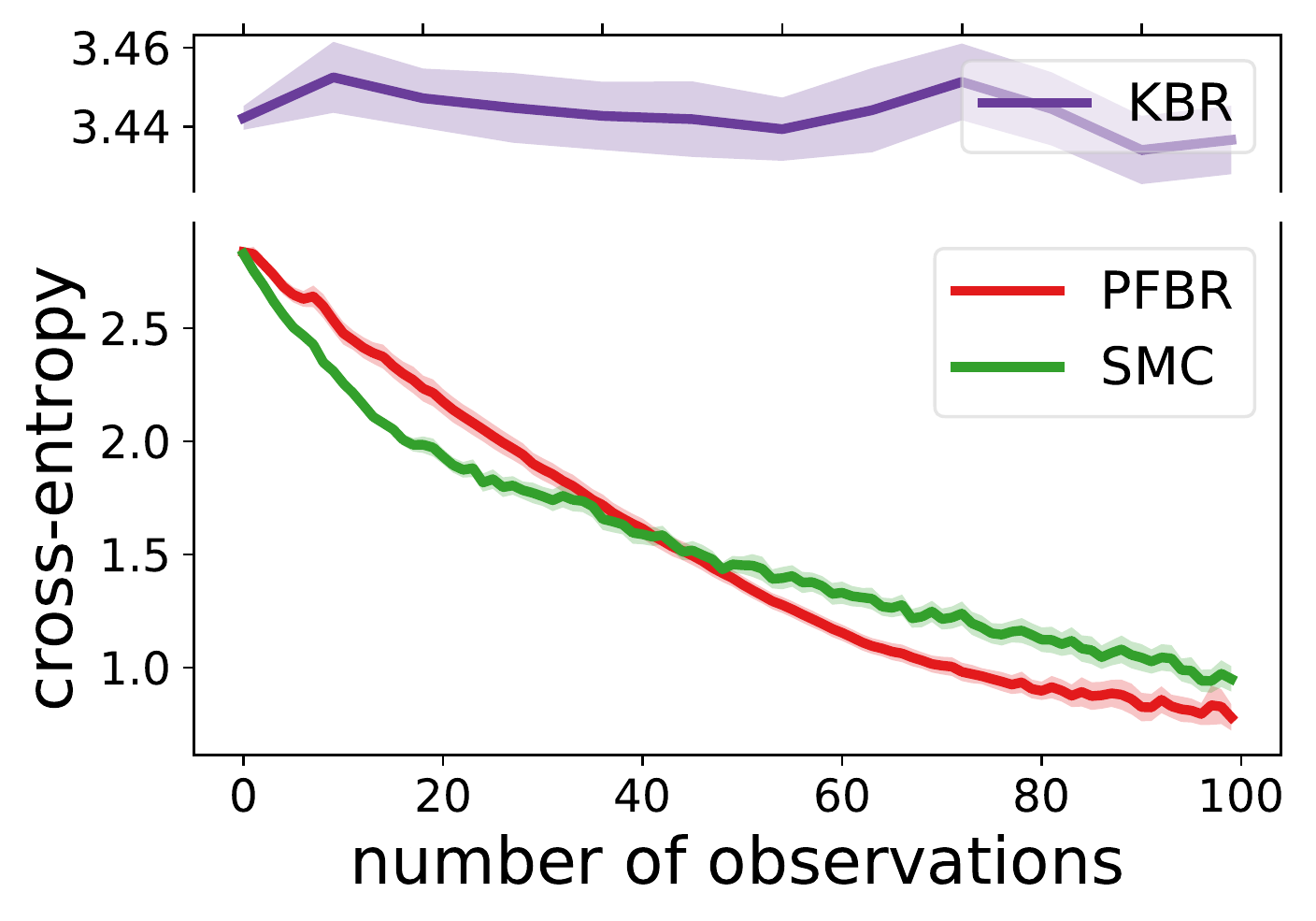} &  \includegraphics[width=0.23\textwidth]{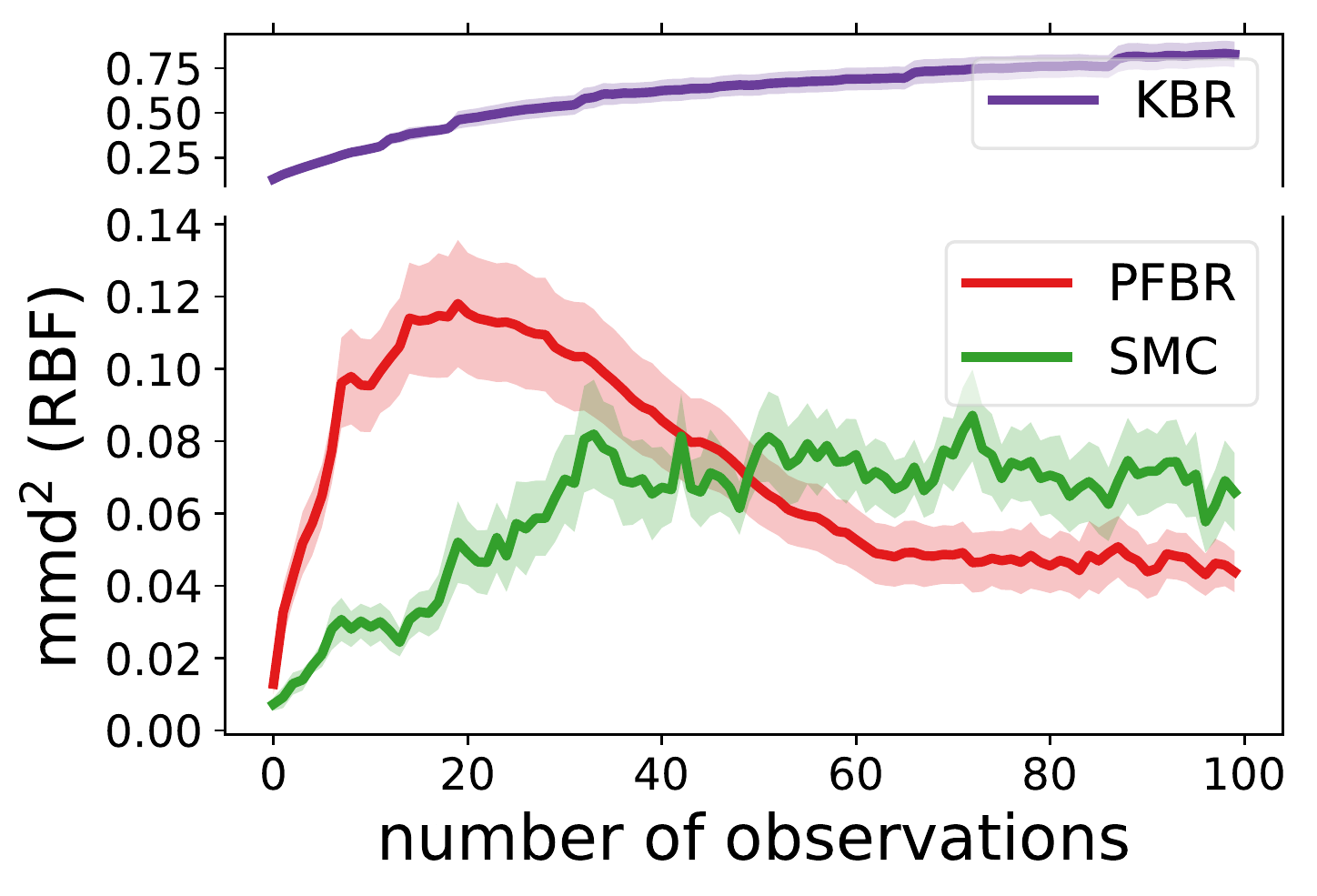} & \includegraphics[width=0.23\textwidth]{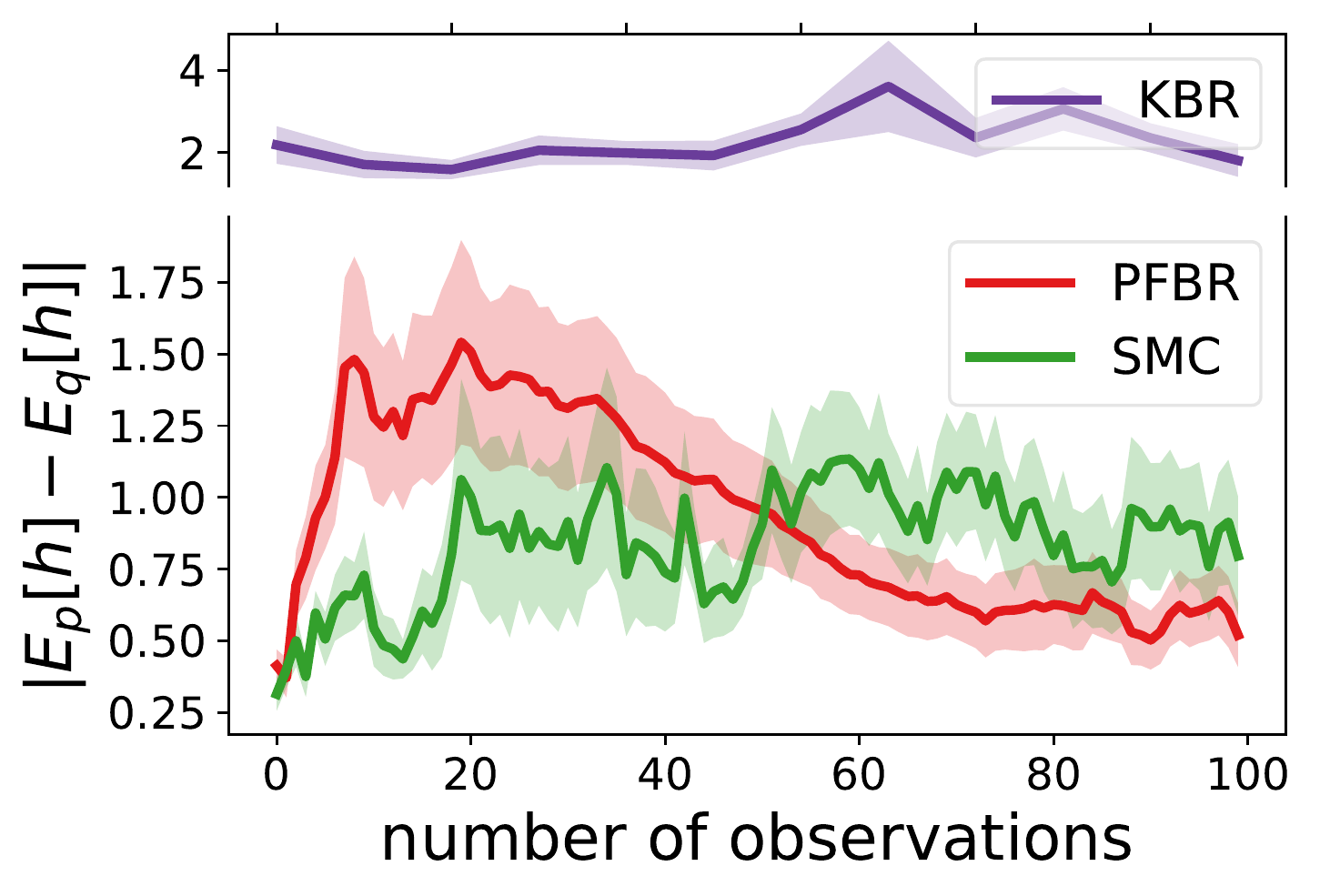}
       \\
       (a) cross-entropy  & (b) MMD$^2$ with RBF kernel&(c) Integral estimation
   \end{tabular}
    \label{fig:apx_mvn}
    \caption{Experimental results on 2 dimensional multivariate Gaussian model.}
\end{figure}

\subsection{LDS Model}
\begin{figure}[h!]
   \begin{tabular}{cccc}
       \includegraphics[width=0.23\textwidth]{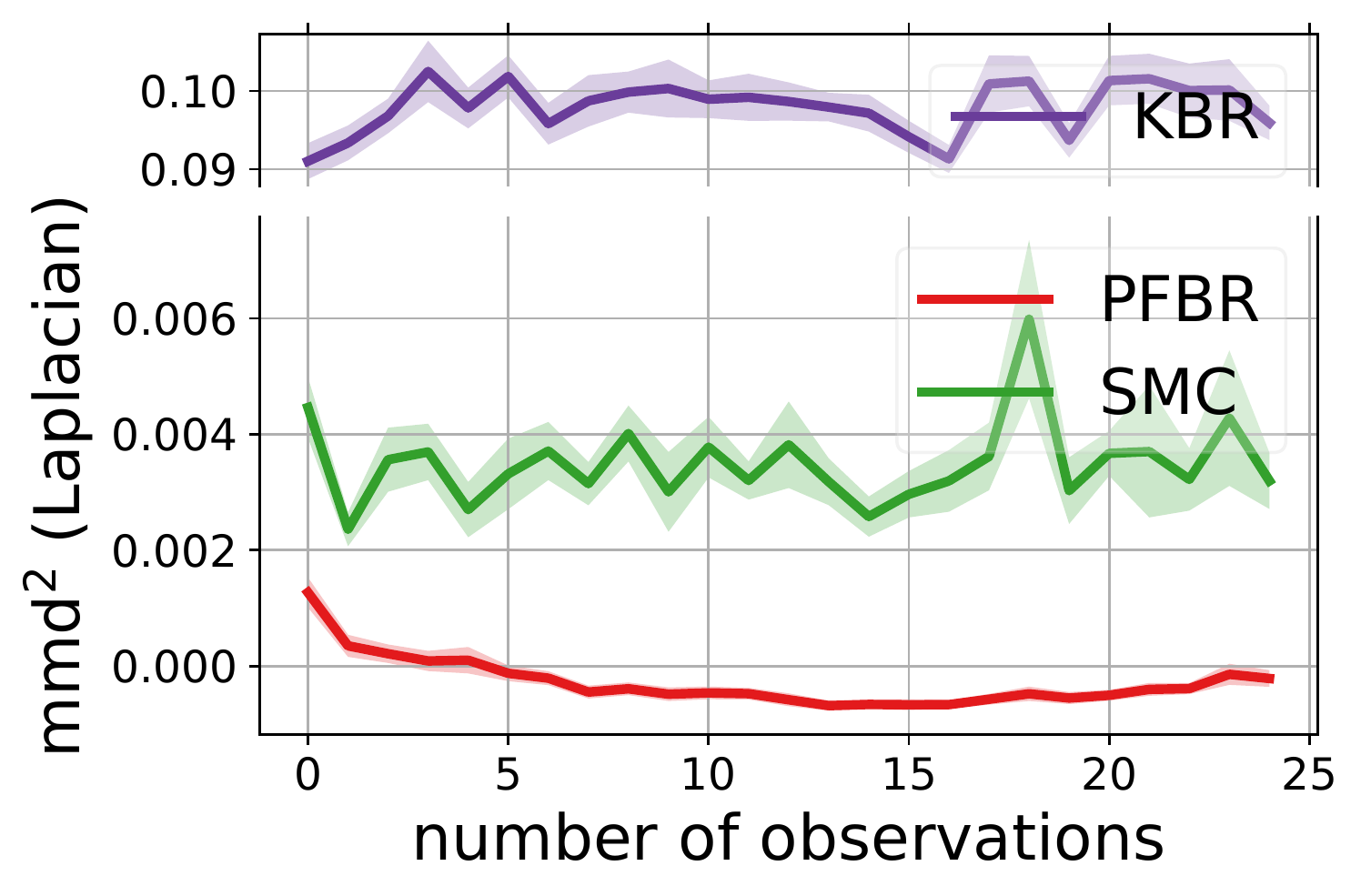} &  \includegraphics[width=0.23\textwidth]{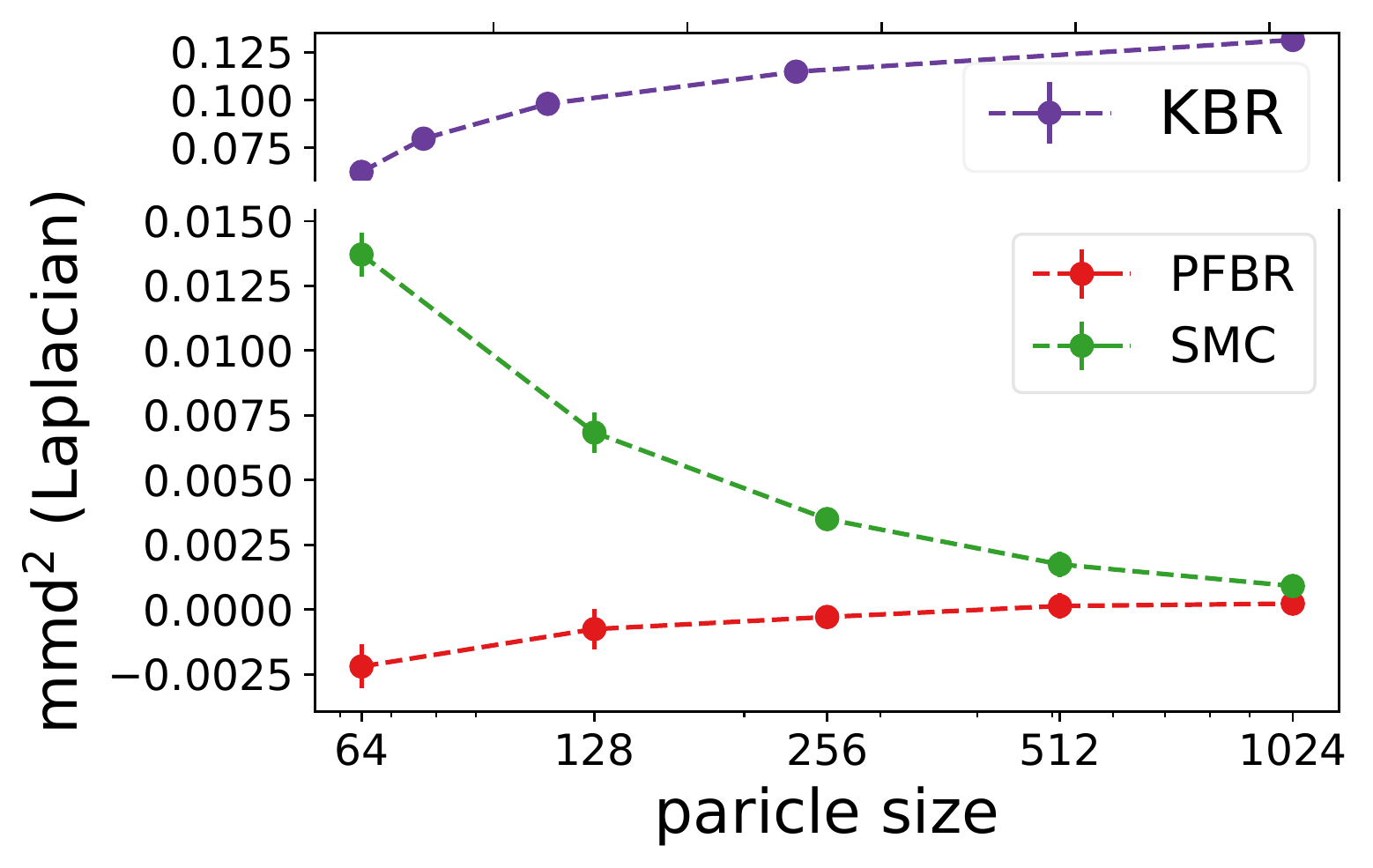} & \includegraphics[width=0.23\textwidth]{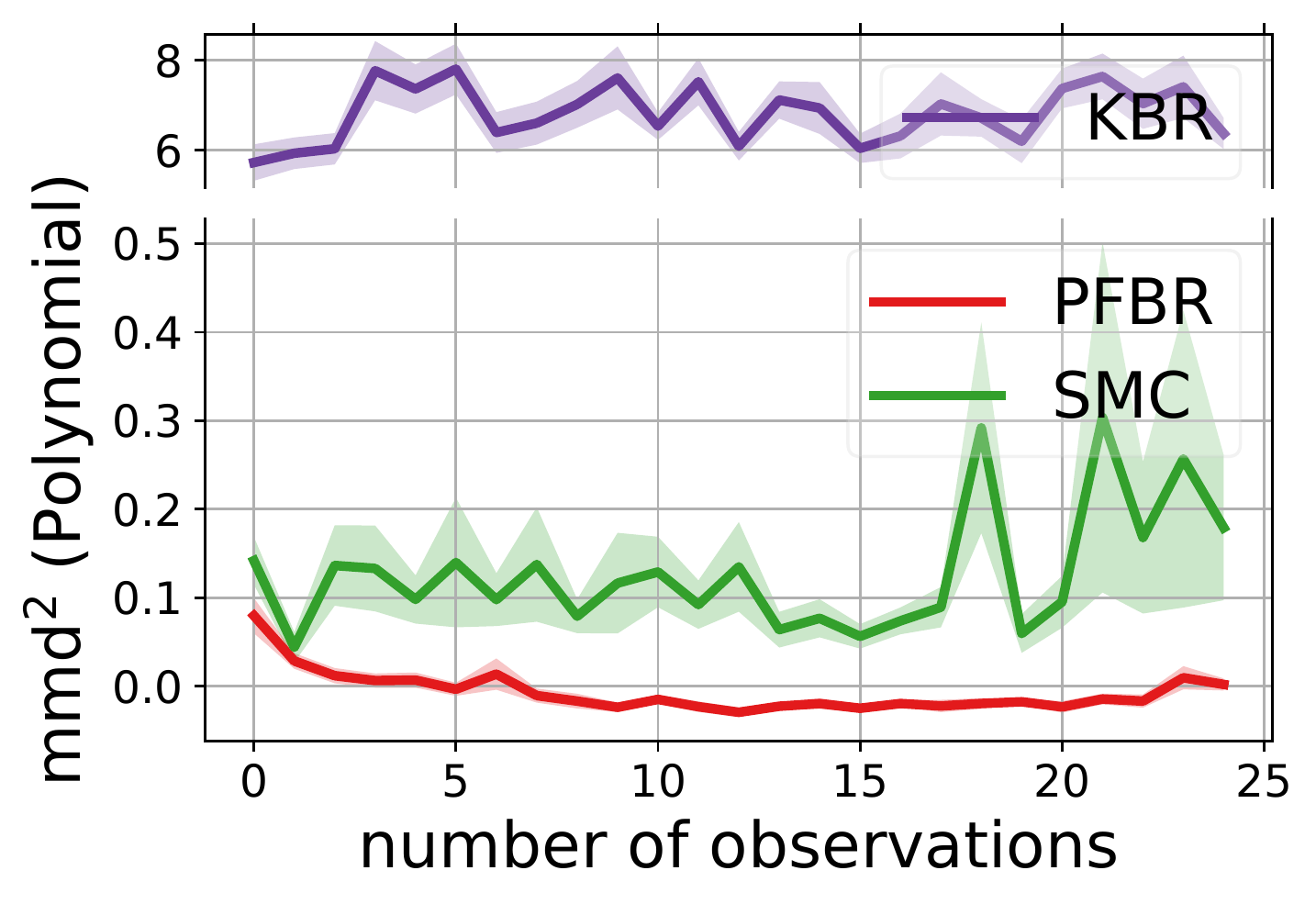} &
       \includegraphics[width=0.23\textwidth]{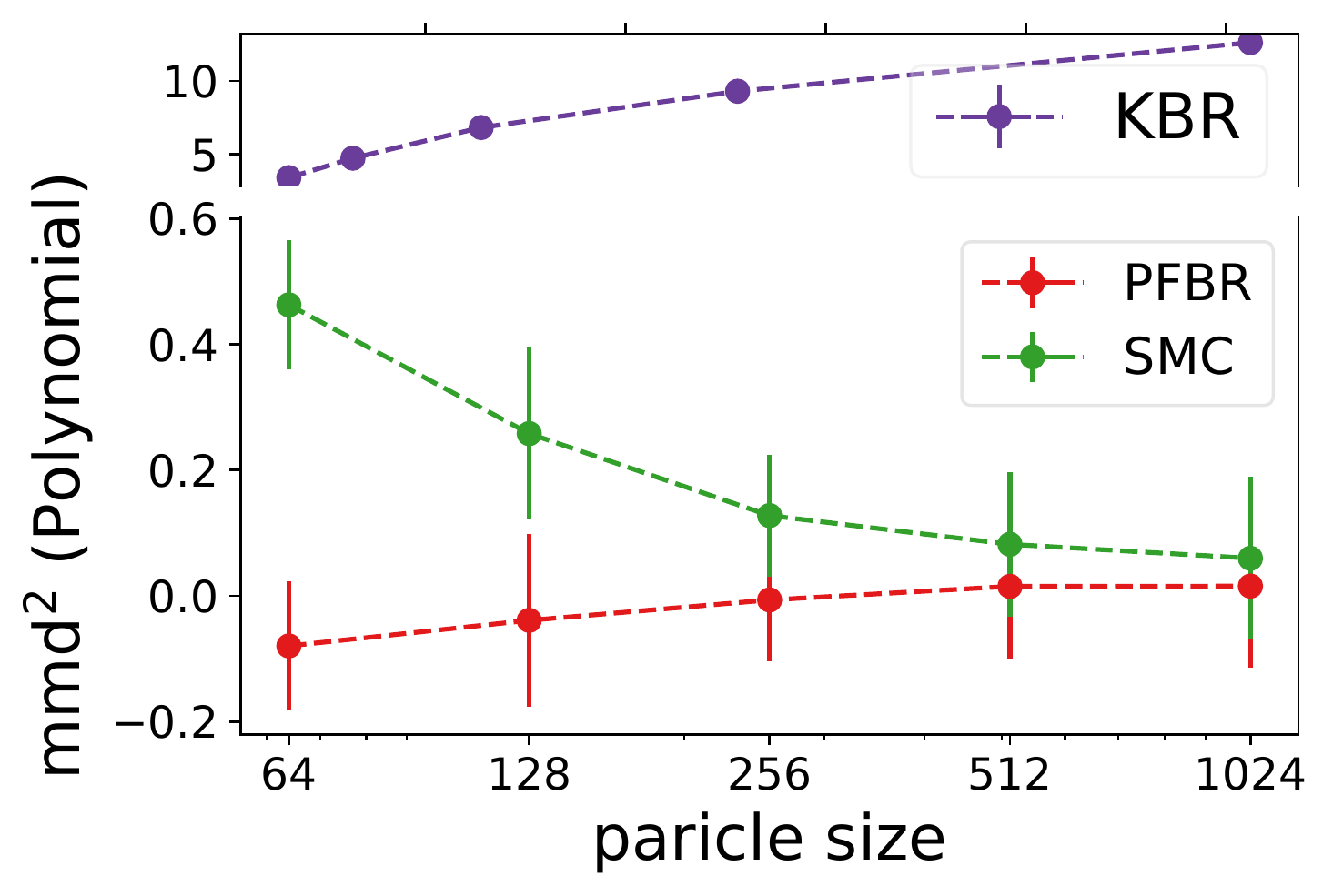}
       \\
       \multicolumn{2}{c}{(a) MMD$^2$ with Laplacian kernel} & 
       \multicolumn{2}{c}{(b) MMD$^2$ with Polynomial kernel} \\
       \includegraphics[width=0.23\textwidth]{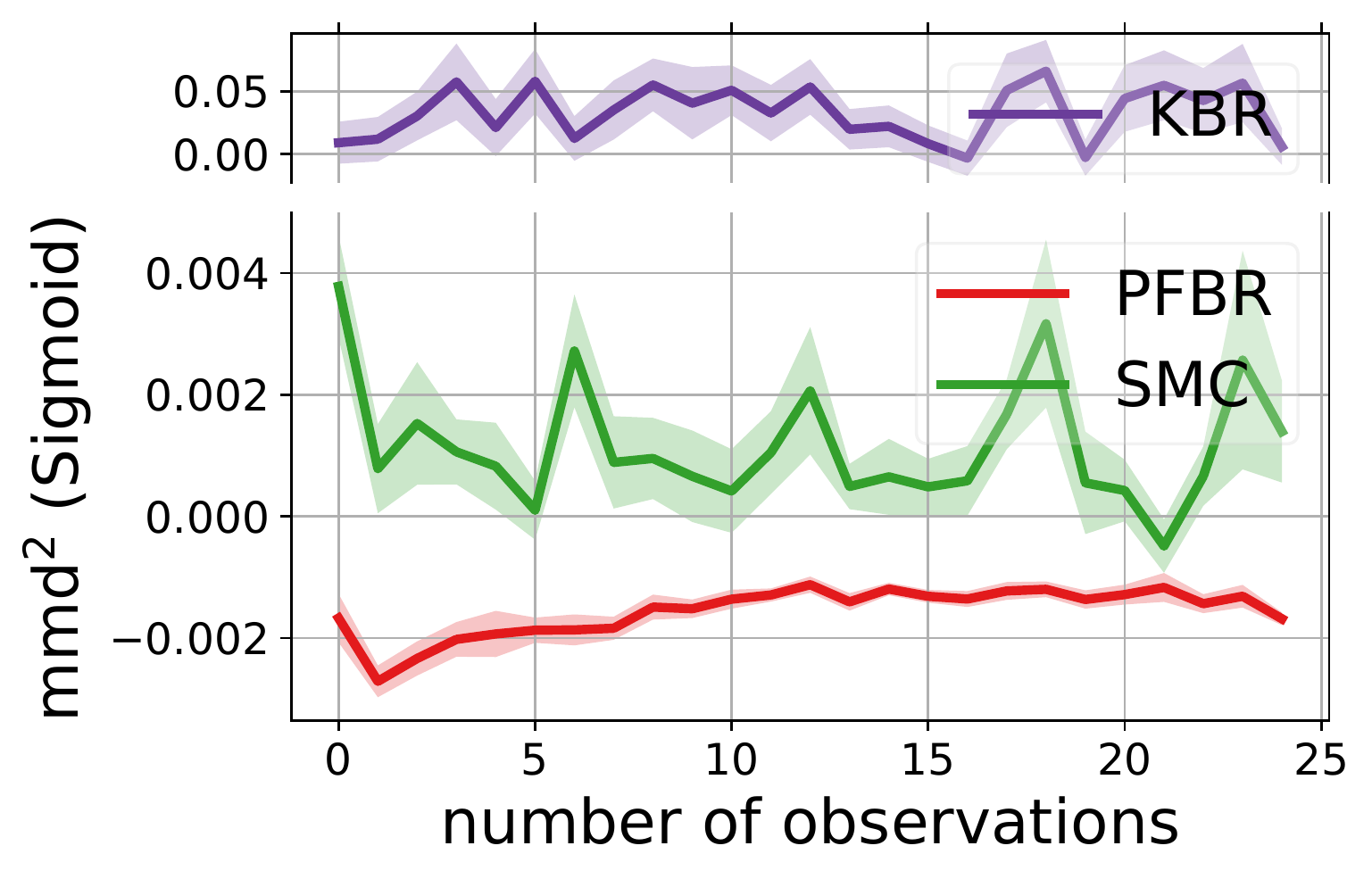}
       & 
       \includegraphics[width=0.23\textwidth]{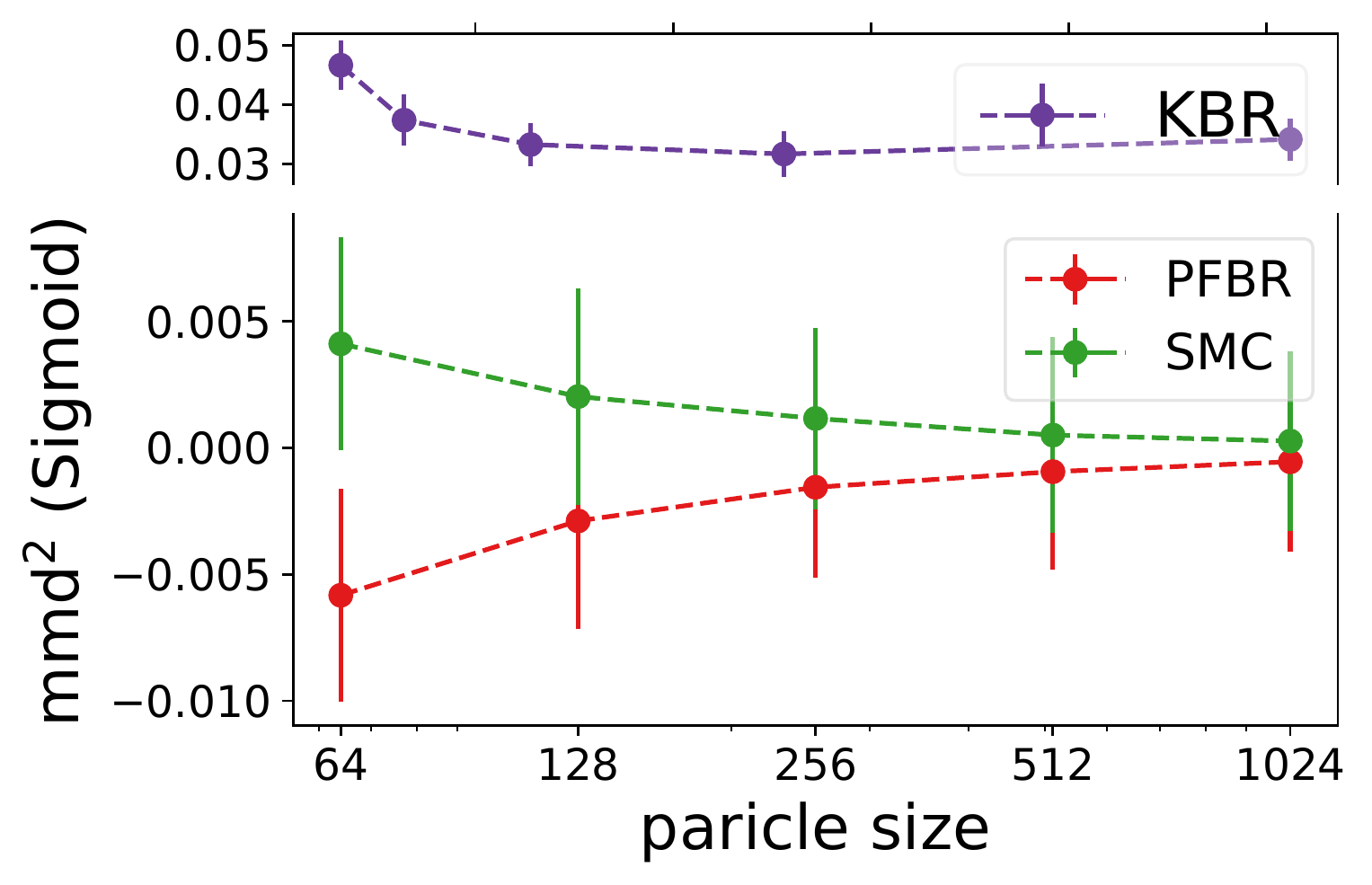}
       &
       \includegraphics[width=0.23\textwidth]{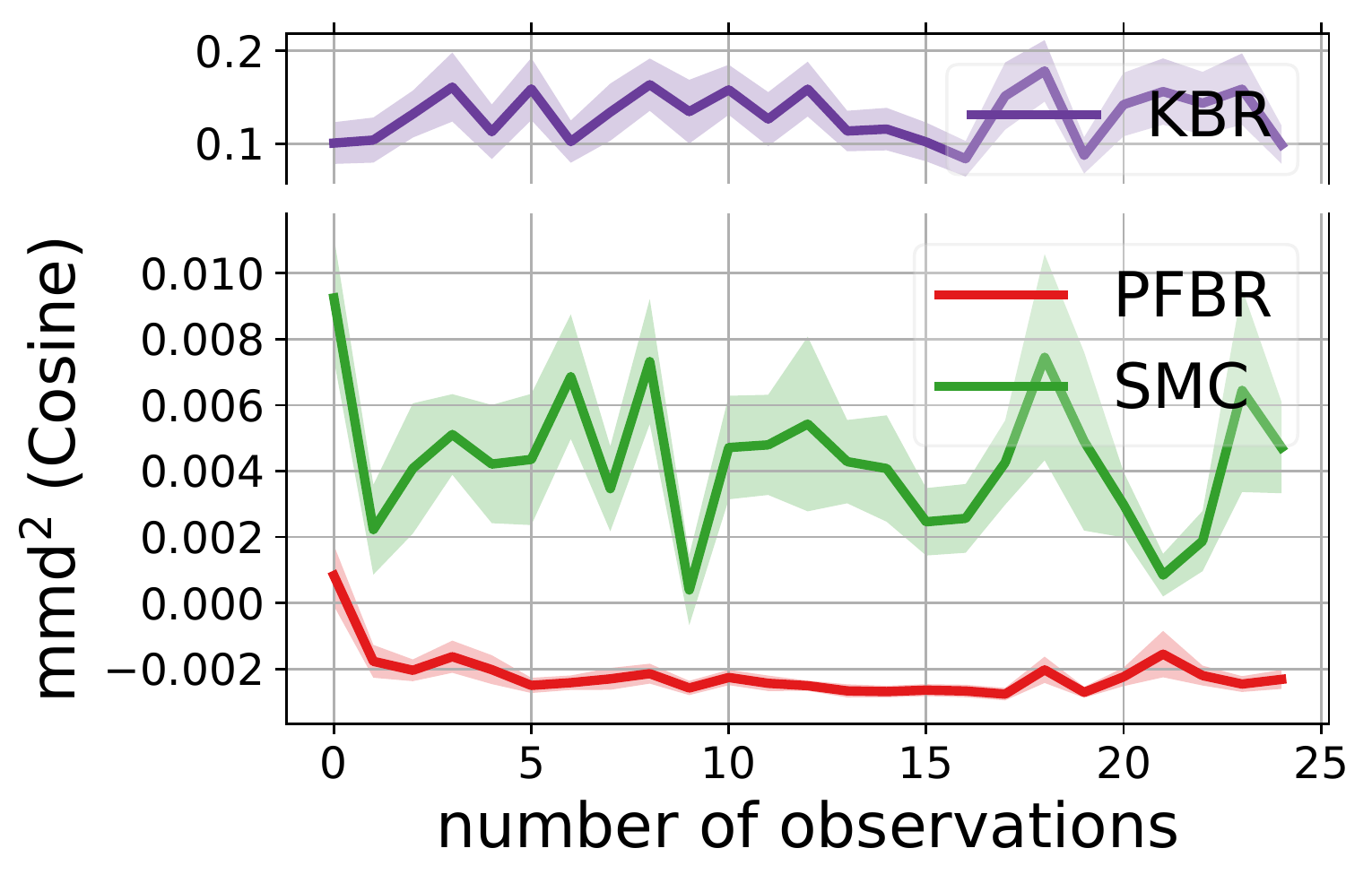}
       &
       \includegraphics[width=0.23\textwidth]{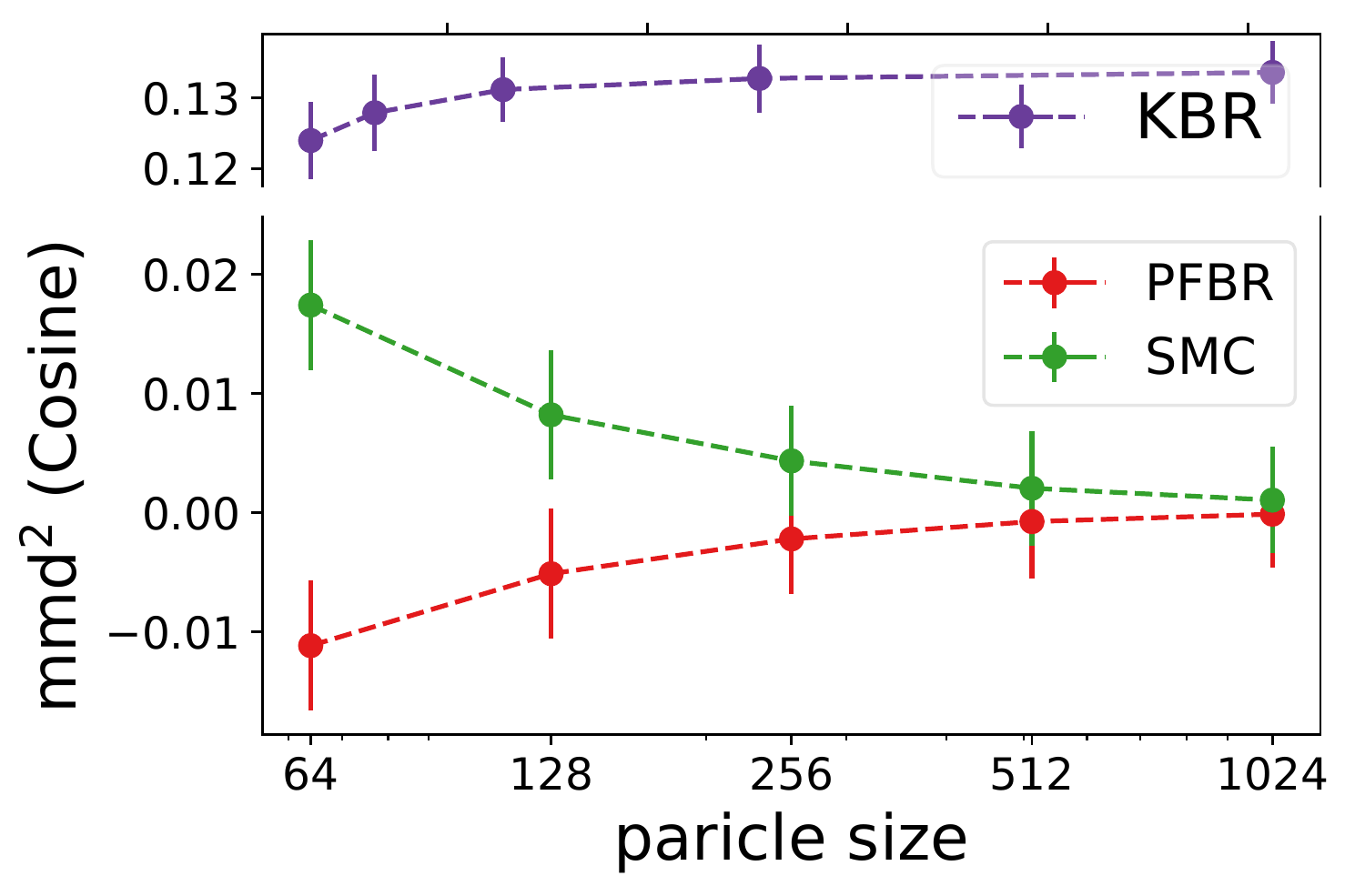} \\
       \multicolumn{2}{c}{(c) MMD$^2$ with Sigmoid kernel} & 
        \multicolumn{2}{c}{(c) MMD$^2$ with Cosine kernel} 
        \\
     \includegraphics[width=0.23\textwidth]{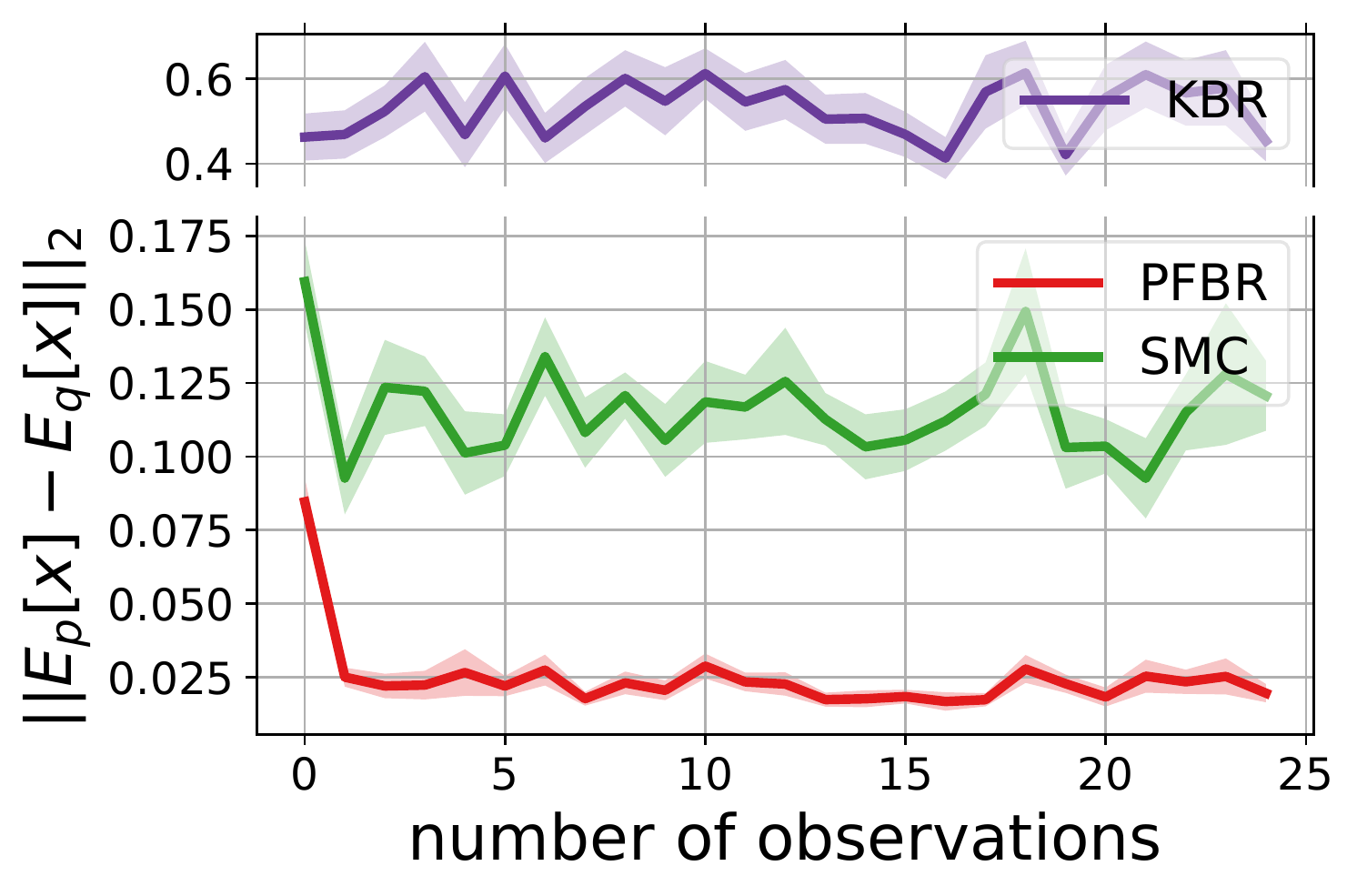}  &
     \includegraphics[width=0.23\textwidth]{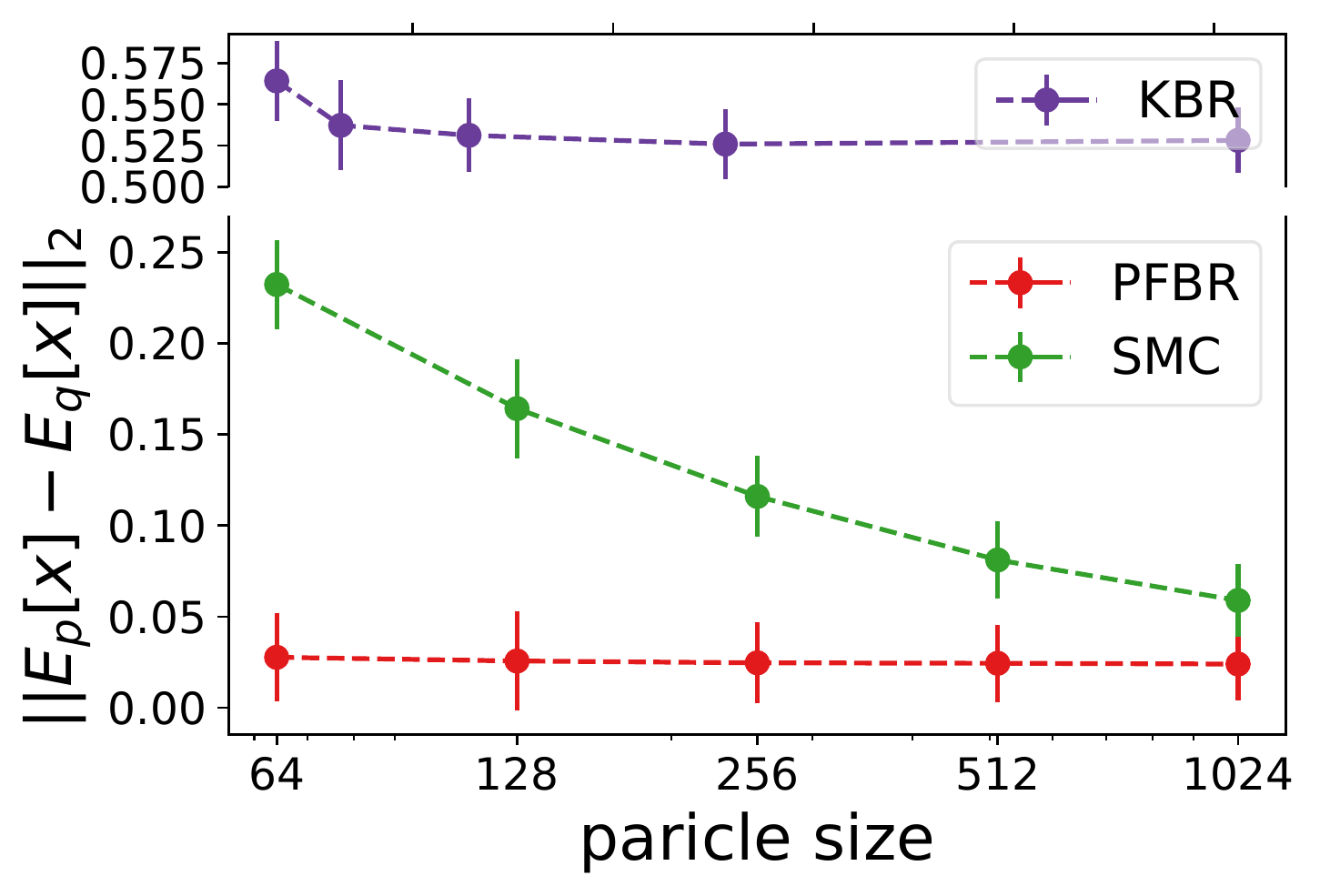} &
     \includegraphics[width=0.23\textwidth]{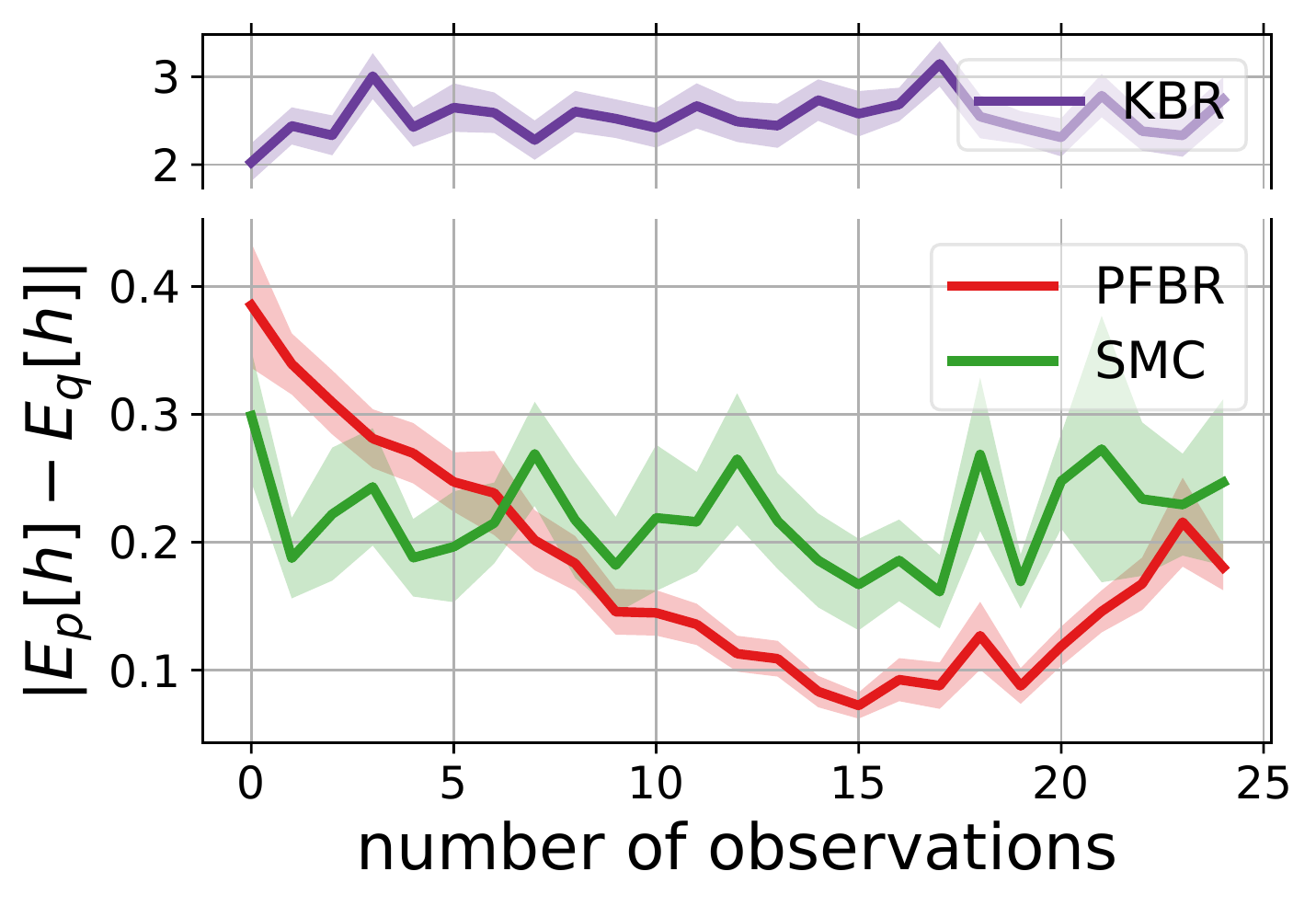} &
     \includegraphics[width=0.23\textwidth]{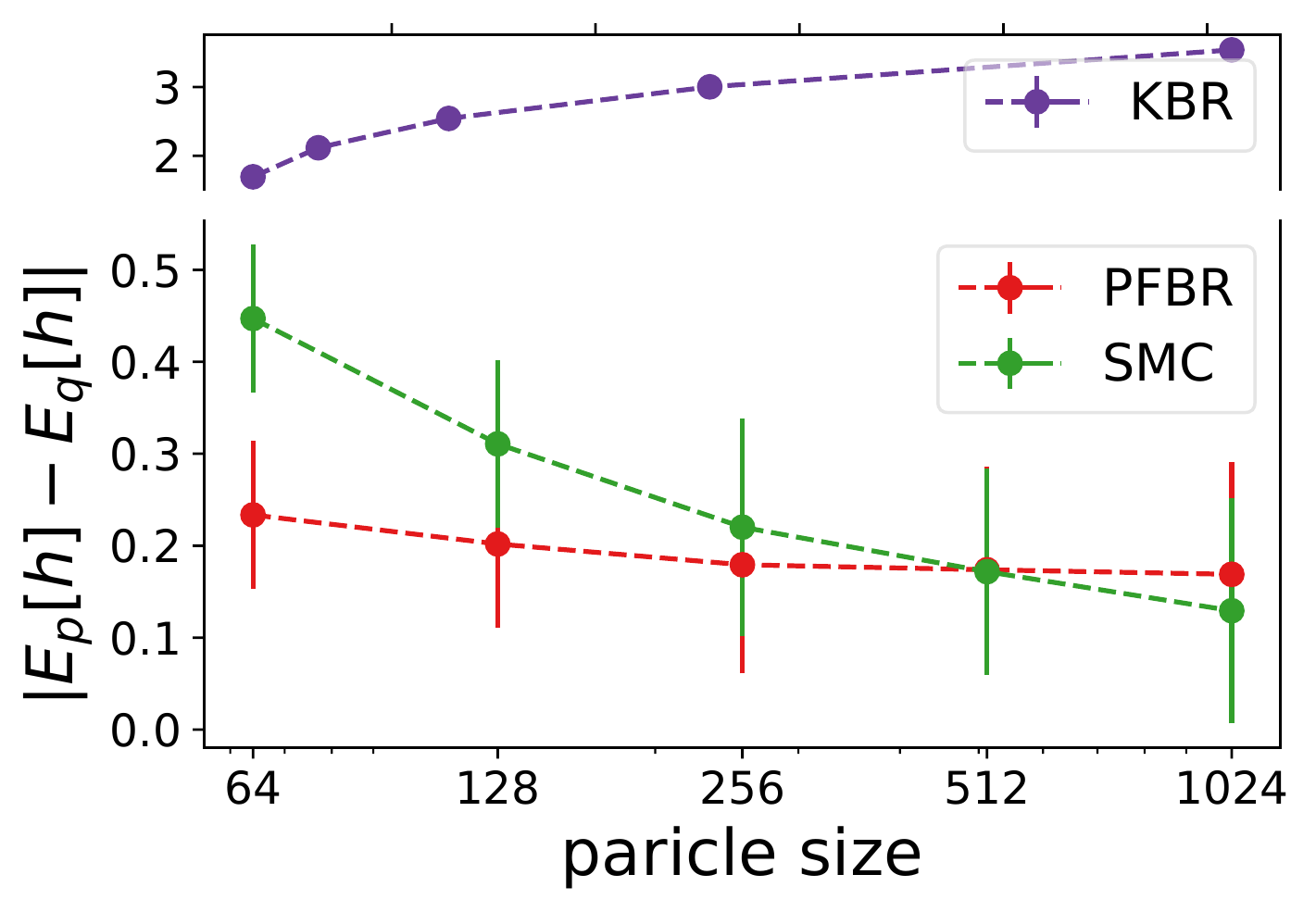} \\
     \multicolumn{2}{c}{(d) Integral estimation on $h(\vx)=\vx$} & \multicolumn{2}{c}{(e) Integral estimation on $h(\vx)=(A\vx+\va)^\top (B\vx+\vb)$}
   \end{tabular}
    \label{fig:apx_lds}
    \caption{Experimental results on LDS model.}
\end{figure}
\fi

\end{document}